%% file: sifi.tex
\documentclass[11pt]{article}
\usepackage[utf8]{inputenc}
\usepackage{amsgen,amsmath,amstext,amsbsy,amsopn,amssymb,bbm,listings,float}
\usepackage{amsthm}
\usepackage{fullpage}

\usepackage{graphicx}
\usepackage{tcolorbox}
\usepackage[font={footnotesize}]{caption}
\usepackage[english]{babel}
\usepackage{amsthm}

\begingroup
\makeatletter
   \@for\theoremstyle:=definition,remark,plain\do{%
     \expandafter\g@addto@macro\csname th@\theoremstyle\endcsname{%
        \addtolength\thm@preskip\parskip
     }%
   }
\endgroup

\usepackage[numbers,sort]{natbib}
\usepackage{algorithm}% http://ctan.org/pkg/algorithm
\usepackage{algpseudocode}
\newtheorem{theorem}{Theorem}[section]

\newtheorem{lemma}[theorem]{Lemma}
\newtheorem{example}[theorem]{Example}
\newtheorem*{remark}{Remark}
\newtheorem{definition}{Definition}[section]

\newtheorem{claim}[theorem]{Claim}
\newtheorem{fact}[theorem]{Fact}
\usepackage{hyperref}
\usepackage{tikz}
\hypersetup{
    colorlinks=true,
    linkcolor=blue,
    filecolor=magenta,
    urlcolor=cyan,
    citecolor = blue
}

\newcommand{\eps}{\varepsilon}
\newcommand{\A}{\mathcal{A}}
\newcommand{\M}{\mathcal{M}}
\newcommand{\X}{\mathcal{X}}
\newcommand{\cX}{\mathcal{X}}
\newcommand{\cD}{\mathcal{D}}
\newcommand{\cV}{\mathcal{V}}
\newcommand{\Y}{\mathcal{Y}}
\newcommand{\Rej}{\mathsf{RejSamp}}
\renewcommand{\P}[2]{\mathbb{P}_{#1}\left[ #2 \right]}
\newcommand{\E}[2]{\mathbb{E}_{#1}\left[ #2 \right]}
\newcommand{\ind}{\perp\!\!\!\!\perp}
\newcommand{\Lap}[1]{\ensuremath{\mathsf{Lap}}\left( #1 \right)}
\newcommand{\Ber}[1]{\ensuremath{\mathsf{Ber}}\left( #1 \right)}
\newcommand{\RR}[1]{\ensuremath{\mathsf{RR}}\left( #1 \right)}
\newcommand{\mpj}[0]{\mathcal{MPJ}}
\newcommand{\tra}{\Pi}
\newcommand{\trar}{\pi}
\newcommand{\kl}[2]{D_{KL}\left(#1 \mid \mid #2\right)}
\newcommand{\tv}[2]{\Vert #1 - #2 \Vert_{TV}}
\newcommand{\sqhel}[2]{h^2\left(#1,#2\right)}
\newcommand{\hel}[2]{h\left(#1,#2\right)}
\newcommand{\expt}{\ensuremath{\mathsf{Expt}}}
\newcommand{\fexpt}{\ensuremath{\mathsf{FollowExpt}}}
\newcommand{\bexpt}{\ensuremath{\mathsf{BayesExpt}}}
\newcommand{\red}{\ensuremath{\mathsf{Reduction}}}
\newcommand{\F}{\mathcal{F}}
\def \epsilon {\varepsilon} % make epsilon consistent

\newif\ifcomment
\commenttrue

\ifcomment
\newcommand{\mj}[1]{\textcolor{magenta}{[MJ: #1]}}
\newcommand{\jm}[1]{\textcolor{cyan}{[JM: #1]}}
\newcommand{\sn}[1]{\textcolor{olive}{[SN: #1]}}
\newcommand{\ar}[1]{\textcolor{violet}{[AR: #1]}}
\else
\newcommand{\mj}[1]{}
\newcommand{\jm}[1]{}
\newcommand{\sn}[1]{}
\newcommand{\ar}[1]{}
\fi

\title{The Role of Interactivity in Local Differential Privacy}
\author{Matthew Joseph \thanks{University of Pennsylvania, Computer and Information Science. \href{mailto:majos@cis.upenn.edu}{\texttt{majos@cis.upenn.edu}}} \and Jieming Mao \thanks{University of Pennsylvania, Warren Center. \href{mailto:jiemingm@seas.upenn.edu}{\texttt{jiemingm@seas.upenn.edu}}} \and Seth Neel \thanks{University of Pennsylvania, Wharton Statistics. \href{mailto:sethneel@wharton.upenn.edu}{\texttt{sethneel@wharton.upenn.edu}}} \and Aaron Roth \thanks{University of Pennsylvania, Computer and Information Science. \href{mailto:aaroth@cis.upenn.edu}{\texttt{aaroth@cis.upenn.edu}}}}
\begin{document}

\maketitle
\begin{abstract}
We study the power of interactivity in local differential privacy. First, we focus on the difference between \emph{fully interactive} and \emph{sequentially interactive} protocols. Sequentially interactive protocols may query users adaptively in sequence, but they cannot return to previously queried users. The vast majority of existing lower bounds for local differential privacy apply only to sequentially interactive protocols, and before this paper it was not known whether fully interactive protocols were more powerful.

We resolve this question. First, we classify locally private protocols by their \emph{compositionality}, the multiplicative factor $k \geq 1$ by which the sum of a protocol's single-round privacy parameters exceeds its overall privacy guarantee. We then show how to efficiently transform any fully interactive $k$-compositional protocol into an equivalent sequentially interactive protocol with an $O(k)$ blowup in sample complexity. Next, we show that our reduction is tight by exhibiting a family of problems such that for any $k$, there is a fully interactive $k$-compositional protocol which solves the problem, while no sequentially interactive protocol can solve the problem without at least an $\tilde \Omega(k)$ factor more examples.

We then turn our attention to hypothesis testing problems. We show that for a large class of compound hypothesis testing problems --- which include all simple hypothesis testing problems as a special case --- a simple noninteractive test is optimal among the class of all (possibly fully interactive) tests.
\end{abstract}

\thispagestyle{empty} \setcounter{page}{0}
\clearpage

\input{intro}
\input{prelims}
\input{si-fi-comp}
\input{lowerbound}
\input{hypothesis}

\appendix
\input{appendix}

\input{prelim-info}

\bibliographystyle{plainnat}
\bibliography{si_fi}

\end{document}

%% file: intro.tex
\section{Introduction}
In the last several years, differential privacy in the \emph{local} model has seen wide adoption in industry, including at Google \cite{EPK14, BEMMR+17}, Apple \cite{A17}, and Microsoft \cite{DKY17}. The choice of adopting the \emph{local} model of differential privacy --- in which privacy protections are added at each individual's device, before data aggregation --- instead of the more powerful \emph{central} model of differential privacy --- in which a trusted intermediary is allowed to first aggregate data before adding privacy protections --- is driven by practical concerns. Local differential privacy frees the data analyst from many of the responsibilities that come with the stewardship of private data, including liability for security breaches, and the legal responsibility to respond to subpoenas for private data, amongst others. However, the local model of differential privacy comes with its own practical difficulties. The most well known of these is the need to have access to a larger number of users than would be necessary in the central model. Another serious obstacle --- the one we study in this paper --- is the need for \emph{interactivity}.

There are two reasons why interactive protocols --- which query users adaptively, as a function of the answers to previous queries --- pose practical difficulties. The first is that communication with user devices is slow: the communication in noninteractive protocols can be fully parallelized, but for interactive protocols, the number of rounds of interactivity becomes a running-time bottleneck. The second is that user devices can go offline or otherwise become unreachable --- and so it may not be possible to return to a previously queried user and pose a new query. The first difficulty motivates the study of noninteractive protocols. The second difficulty motives the study of \emph{sequentially interactive} protocols \cite{DJW13} --- which may pose adaptively chosen queries --- but must not pose more than one query to any user (and so in particular never need to return to a previously queried user).

It has been known since \cite{KLNRS11} that there can be an exponential gap in the sample complexity between noninteractive and interactive protocols in the local model of differential privacy, and that this gap can manifest itself even in natural problems like convex optimization \cite{STU17, U18}. However, it was not known whether the full power of the local model could be realized with only \emph{sequentially interactive} protocols. Almost all known lower bound techniques applied only to either noninteractive or sequentially interactive protocols, but there were no known fully interactive protocols that could circumvent lower bounds for sequential interactivity.

\subsection{Our Results}

We present two kinds of results, relating to the power of sequentially adaptive protocols and non-adaptive protocols respectively. Throughout, we consider protocols operating on datasets that are drawn i.i.d. from some unknown distribution $\cD$, and focus on the \emph{sample complexity} of these protocols: how many users (each corresponding to a sample from $\cD$) are needed in order to solve some problem, defined in terms of $\cD$.
\paragraph{Sequential Interactivity}{We classify locally private protocols in terms of their \emph{compositionality}. Informally, a protocol is $k$-compositional if the privacy costs $\{\epsilon^i_j\}_{j=1}^r$ of the local randomizers executed by any user $i$ over the course of the protocol sum to at most $k\epsilon$, where $\epsilon$ is the overall privacy cost of the protocol: $\sum_j \epsilon^i_j \leq k\epsilon$. When $k = 1$, we say that the protocol is compositional. Compositional protocols capture most of the algorithms studied in the published literature, and in particular, any protocol whose privacy guarantee is proven using the composition theorem for $\epsilon$-differential privacy\footnote{Not every protocol is $1$-compositional: exceptions include RAPPOR~\cite{EPK14} and the evolving data protocol of Joseph et al. \cite{JRUW18}.}.
\begin{enumerate}
\item \textbf{Upper Bounds}: For any (potentially fully interactive) compositional protocol $M$, we give a generic and efficient reduction that compiles it into a sequentially interactive protocol $M'$, with only a constant factor blow-up in privacy guarantees and sample complexity, while preserving (exactly) the distribution on transcripts generated. This in particular implies that up to constant factors, sequentially adaptive compositional protocols are as powerful as fully adaptive compositional protocols.  More generally, our reduction compiles an arbitrary $k$-compositional protocol $M$ into a sequentially interactive protocol $M'$ with the same transcript distribution, and a blowup in sample complexity of $O(k)$.
\item \textbf{Lower Bounds}: We show that our upper bound is tight by proving a separation between the power of sequentially and fully interactive protocols in the local model. In particular, we define a family of problems (Multi-Party Pointer Jumping) such that for any $k$, there is a fully interactive $k$-compositional protocol which can solve the problem given sample complexity $n = n(k)$, but such that no sequentially interactive protocol with the same privacy guarantees can solve the problem with sample complexity $\tilde o(k\cdot n)$. Thus, the sample complexity blowup of our reduction cannot be improved in general.
\end{enumerate}}

\paragraph{Noninteractivity}{
We then turn our attention to the power of noninteractive protocols. We consider a large class of compound hypothesis testing problems --- those such that both the null hypothesis $H_0$ and the alternative hypothesis $H_1$ are closed under mixtures. For every problem in this class, we show that the optimal locally private hypothesis test is noninteractive. We do this by demonstrating the existence of a simple hypothesis test for such problems. We then prove that this test's sample complexity is optimal even among the set of all fully interactive tests by extending information theoretical lower bound techniques developed by~\citet{BGMNW16} and first applied to local privacy by~\citet{JKMW18} and~\citet{DR19} to the fully interactive setting.
}

\subsection{Related Work}
The local model of differential privacy was introduced by ~\citet{DMNS06} and further formalized by \citet{KLNRS11}, who also gave the first separation between noninteractive locally private protocols and interactive locally private protocols. They did so by constructing a problem, Masked Parity, that requires exponentially larger sample complexity without interaction than with interaction.~\citet{DF18} later expanded this result to a different, larger class of problems.~\citet{STU17} proved a similar separation between noninteractive and interactive locally private convex optimization protocols that use neighborhood-based oracles.

Recent work by Acharya et al.~\cite{ACFT19, ACT19} gives a qualitatively different separation between the private-coin and public-coin models of noninteractive local privacy. Informally, the public-coin model allows for an additional ``half step'' of interaction over the private-coin model in the form of coordinated local randomizer choices  across users. In this paper, we use the public-coin model of noninteractivity.

Duchi et al.~\cite{DJW13} introduced the notion of sequential interactivity for local privacy. They also provided the first general techniques for proving lower bounds for sequentially interactive locally private protocols by bounding the KL-divergence between the output distributions of $\eps$-locally private protocols with different input distributions as a function of $\eps$ and the total variation distance between these input distributions. Bassily and Smith~\cite{BS15} and Bun et al.~\cite{BBNS19} later generalized this result to $(\eps,\delta)$-locally private protocols, and Duchi et al.~\cite{DJW18} obtained an analogue of Assouad's method for proving lower bounds for sequentially interactive locally private protocols.

More recently, Duchi and Rogers~\cite{DR19} showed how to combine the above analogue of Assouad's method with techniques from information complexity~\cite{GMN14, BGMNW16} to prove lower bounds for estimation problems that apply to a restricted class of \emph{fully} interactive locally private protocols. A corollary of their lower bounds is that several known \emph{noninteractive} algorithms are optimal minimax estimators within the class they consider. However, their results do not imply any separation between sequential and full interaction. Moreover, our reduction implies that every (arbitrarily interactive) compositional locally private algorithm can be reduced to a sequentially interactive protocol with only constant blowup in sample complexity, and as a result all known lower bounds for sequentially interactive protocols also hold for arbitrary compositional protocols.

\citet{CKMSU18} study simple hypothesis testing under the centralized model of differential privacy, and Theorem 1 of~\citet{DJW13} implies a tight lower bound for sequentially interactive locally private simple hypothesis testing. We extend this lower bound to the fully interactive setting and match it with a noninteractive upper bound for a more general class of compound testing problems that includes simple hypothesis testing as a special case. 

Finally, recent subsequent work~\cite{JMR20} gives a stronger exponential sample complexity separation separation between the sequentially and fully interactive models. It does so through a general connection between communication complexity and sequentially interactive sample complexity. Applying this connection to a communication problem similar to the ``multi-party pointer jumping'' described in Section~\ref{sec:lb} completes the result.

%% file: prelims.tex
\section{Preliminaries}
\label{sec:prelims}
We begin with the definition of approximate differential privacy. Given data domain $\cX$, two data sets $S, S' \in \X^n$ are \emph{neighbors} (denoted $S \sim S'$) if they differ in at most one coordinate: i.e. if there exists an index $i$ such that for all $j \neq i$, $S_j = S'_j$. A differentially private algorithm must have similar output distributions on all pairs of neighboring datasets.

\begin{definition}[\cite{DMNS06}]
	Let $\eps,\delta \geq 0$. A randomized algorithm $\M:\X^n\rightarrow \mathcal{O}$ is \emph{$(\eps,\delta)$-differentially private} if for every pair of neighboring data sets $S \sim S' \in \X^n$, and every event $\Omega \subseteq \mathcal{O}$ $$\P{\M}{\A(S) \in \Omega} \leq \exp(\eps)\P{\M}{\A(S') \in \Omega} + \delta.$$ When $\delta = 0$, we say that $\M$ satisfies (pure) $\eps$-differential privacy.
\end{definition}

Differential privacy has two nice properties. First, it composes neatly: the composition of algorithms $\M_1, \ldots, \M_n$ that are respectively $(\eps_1, \delta_1), \ldots, (\eps_n, \delta_n)$-differentially private is $(\sum_i \eps _i, \sum_i \delta_i)$-differentially private. For pure differential privacy, this is tight in general. Second, differential privacy is resilient to post-processing: given an $(\eps,\delta)$-differentially private $\M$ and any function $f$, $f \circ \M$ is still $(\eps,\delta)$-differentially private (see Appendix~\ref{subsec:dp_props} for details). For brevity, we often abbreviate ``differential privacy'' as ``privacy''.

As defined, the constraint of differential privacy is on the \emph{output} of an algorithm $\M$, not on its internal workings. Hence, it implicitly assumes a trusted data curator, who has access to the entire raw dataset. This is sometimes referred to as differential privacy in the central model. In contrast, this paper focuses on the more restrictive \emph{local} model~\cite{DMNS06} of differential privacy. In the local model, the private computation is an interaction between $n$ users, each of whom hold exactly one dataset record, and is coordinated by a protocol $\A$. We assume throughout this paper that each user's datum is drawn i.i.d. from some unknown distribution: $x_i \sim_{iid} \cD$\footnote{Roughly speaking, this corresponds to a setting in which users are ``symmetric'' and in which nothing differentiates them a priori. All of our results generalize to the setting in which there are different ``types'' of users, known to the protocol up front.}. Informally, at each round $t$ of the interaction, a protocol $\A$ observes the transcript of interactions so far, selects a user, and assigns the user a randomizer. The user then applies the randomizer to their datum, using fresh randomness for each application, and publishes the output. In turn, the protocol observes the updated transcript, selects a new user-randomizer pair, and the process continues. We define these terms precisely below.

\begin{definition}
	An \emph{$(\eps,\delta)$-randomizer} $R \colon X \to Y$ is an $(\eps,\delta)$-differentially private function taking a single data point as input.
\end{definition}

A simple, canonical, and useful randomizer is \emph{randomized response}~\cite{W65, DMNS06}.

\begin{example}[Randomized Response]
\label{ex:rr}
	Given data universe $\cX = [k]$ and datum $x_i \in \cX$, $\eps$-randomizer $\RR{x_i, \eps}$ outputs $x_i$ with probability $\tfrac{e^\eps}{e^\eps + k-1}$ and otherwise outputs a uniformly random element of $\cX - \{x_i\}$.
\end{example}

Next, we formally define transcripts and protocols.

\begin{definition}
	A \emph{transcript} $\trar$ is a vector consisting of 5-tuples $(i^t, R_t, \eps_t, \delta_t, y_t)$ --- encoding the user chosen, randomizer assigned, randomizer privacy parameters, and randomized output produced --- for each round $t$. $\trar_{<t}$ denotes the transcript prefix before round $t$. Letting $S_\pi$ denote the collection of all transcripts and $S_R$ the collection of all randomizers, a \emph{protocol} is a function $\A \colon S_\pi \to \left([n] \times S_R \times \mathbb{R}_{\geq 0} \times \mathbb{R}_{\geq 0}\right) \cup \{\perp\}$ mapping transcripts to users, randomizers, and randomizer privacy parameters ($\perp$ is a special character indicating a protocol halt).
\end{definition}

The transcript that results from running a locally private computation will often be post-processed to compute some useful function of the data. However, the privacy guarantee must hold even if the entire transcript is observed. Hence, in this paper we abstract away the task that the computation is intended to solve, and view the output of a locally private computation as simply the transcript it generates.

To clarify the role of interaction in these private computations -- especially when analyzing reductions between computations with different kinds of interactivity -- it is often useful to speak separately of protocols and \emph{experiments}. While the protocol $\A$ is a function mapping transcripts to users and randomizers, the experiment is the interactive process that maps a protocol and collection of users drawn from a distribution $\cD$ to a finished transcript. In the simplest case, \fexpt~(Algorithm~\ref{alg:fexpt}), the experiment exactly follows the outputs of its protocol.

\begin{algorithm}
\caption{}\label{alg:fexpt}
\begin{algorithmic}[1]
\Procedure{$\fexpt$}{$\A, \cD, n$}
\State Draw $n$ users $\{x_i\} \sim \cD^{n}$
\State Initialize transcript $\trar_0 \gets \emptyset$
\For{$t = 1, 2, \ldots $}
	\If{$\A(\trar_{<t}) = \perp$}
		\State Output transcript $\trar_{<t}$
	\Else
		\State $(i^t, R_t, \eps_t, \delta_t) \gets \A(\trar_{<t})$
		\State User $i^t$ publishes $y_t \sim R_t(x_{i^t}, \eps_t, \delta_t)$
	\EndIf
\EndFor
\EndProcedure
\end{algorithmic}
\end{algorithm}

However, experiments may in general heed, modify, or ignore the outputs of their input protocol. We delineate the privacy characteristics of experiment-protocol pairs and protocols in isolation below. Here and throughout, the dataset is not viewed as an input to an experiment, but is drawn from $\cD$ by the experiment-protocol pair. Drawing a fresh user $\sim \cD$ corresponds to adding an additional data point, and so the sample complexity of an experiment-protocol pair is the number of draws from $\cD$ over the run of the algorithm. For the simple algorithm $\fexpt(\A)$ defined above, the sample complexity is always $n$. Finally we remark that although the distribution $\cD$ and the sample complexity $n$ are inputs to the experiment, for brevity we typically omit them and focus on the protocol $\A$; e.g. writing $\expt(\A)$ rather than $\expt(\A, \cD, n)$.
\begin{definition}
	Experiment-protocol pair $\expt(\A)$ satisfies \emph{$(\eps,\delta)$-local differential privacy} (LDP) if it is $(\eps,\delta)$-differentially private in its transcript outputs. A protocol $\A$ satisfies $(\eps,\delta)$-local differential privacy (LDP) if experiment-protocol pair $\fexpt (\A)$ is $(\eps,\delta)$-locally differentially private.
\end{definition}

Experiment-protocol pairs can be, by increasing order of generality, \emph{noninteractive}, \emph{sequentially interactive}, and \emph{fully interactive}.

\begin{definition}
	An experiment-protocol pair $\expt(\A)$ is \emph{noninteractive} if, at each round $t$, as random variables, $(i^t, R_t, \eps_t, \delta_t) \ind \Pi_{<t} \mid t$.
\end{definition}

In other words, noninteractivity forces nonadaptivity, and all user-randomizer assignments are made before the experiment begins. In contrast, in sequentially interactive experiment-protocol pairs, users may be queried adaptively, but only once.

\begin{definition}
	An experiment-protocol pair $\expt(\A)$ is \emph{sequentially interactive} if, at each round $t$, $i^t \neq i^{t-1}, \ldots, i^1$.
\end{definition}

Finally, in in fully interactive experiments, the experiment-protocol may make user-randomizer assignments adaptively, and each user may receive arbitrarily many randomizer assignments. Along the same lines, we say a protocol $\A$ is noninteractive (respectively sequentially and fully interactive) if $\fexpt(\A)$ is a noninteractive (respectively sequentially and fully interactive) experiment-protocol pair. This experiment-protocol formalism will be useful in constructing the full-to-sequential reduction in Section~\ref{sec:simulation}; elsewhere, we typically elide the distinction and simply reason about $\fexpt(\A)$ as ``protocol $\A$''. For any locally private protocol, we refer to the number of users $n$ that it queries as its \emph{sample complexity}. For fully interactive protocols, the total number of rounds --- which we denote by $T$ --- may greatly exceed $n$. In contrast, for both non-interactive and sequentially interactive protocols, the number of rounds $T \leq n$.

At each round $t$ of a fully interactive $\epsilon$-locally private protocol, we know that $\epsilon_t \leq \epsilon$. For many protocols, we can say more about how the $\epsilon_t$ parameters relate to $\epsilon$:

\begin{definition}
	Consider an $\epsilon$-locally private protocol $\A$. Let $\{\epsilon_t\}_{t=1}^T$ denote the minimal privacy parameters of the local randomizers $R_t$ selected at round $t$ considered as random variables. We say the protocol $\A$ is \emph{$k$-compositionally private} if for all $i \in [n]$, with probability $1$ over the randomness of the transcript, $$\sum_{t \colon i_t = i}\epsilon_t \leq k\epsilon.$$ If $k = 1$, a protocol is simply \emph{compositional private}.
\end{definition}
\begin{remark}
In fact, all of our results hold without modification even under the weaker condition of \emph{average} $k$-compositionality. For a protocol $\A$ with sample complexity $n$, $\A$ is $k$-compositional on average if
$$\sum_{t} \epsilon_t \leq k\epsilon n.$$ For brevity, we often shorthand ``$k$-compositionally private'' as simply ``$k$-compositional''.
\end{remark}

Informally, a compositionally private protocol is one in which the privacy parameters for each user ``just add up.'' Almost every locally private protocol studied in the literature (and in particular, every protocol whose privacy analysis follows from the composition theorem for pure differential privacy) is compositionally private\footnote{This simple compositionality applies even if $\{\eps_t\}_{t=1}^T$ are chosen adaptively in each round (see Theorem 3.6 in~\citet{RRUV16}).}. They are so ubiquitous that it is tempting to guess that all $(\epsilon,0)$-locally private protocols are compositional. However, this is false: for every $k$ and $\epsilon$, there are $\epsilon$-locally private protocols that fail to be $k$-compositionally private. The following example shows that by taking advantage of special structure in the data domain and choice of randomizers it is possible to achieve $(\eps,0)$-local privacy, even as the sum of the round-by-round privacy parameters greatly exceeds $\epsilon$.

\begin{example}[Informal]
\label{ex:comp}
	Let the data universe $\cX$ consist of the canonical basis vectors $e_1, \ldots, e_d \in \{0,1\}^{d}$, and let each $x_1, \ldots, x_n$ be an arbitrary element of $\cX$. Consider the $d$ round protocol where, for each round $j \in [d]$, every user $i$ with $x_i = e_j$ outputs a sample from $\RR{1, \eps}$, and the remaining users output a sample from $\Ber{0.5}$. As $\mathsf{RR}(\cdot,\eps)$ is an $\eps$-local randomizer which each user employs only once, and remaining outputs are data-independent, this protocol is $\eps$-locally private.  But the protocol fails to be  $k$-compositionally private for $k < d/2$.
\end{example}

The preceding example demonstrates that the careful choice of local randomizers based on the data universe structure can strongly violate compositional privacy. Seen another way, when multiple queries are asked of the same user, there are situations in which the correlation in privatized responses induced by being run on the same data element can lead to arbitrarily sub-compositional privacy costs. The main result of our paper is that the additional power of a fully interactive protocol, on top of sequential interactivity, is characterized by its compositionality. 

%% file: si-fi-comp.tex
\section{From Full to Sequential Interactivity} \label{sec:simulation}
We show that any $(\eps, 0)$-locally private compositional protocol is ``equivalent" to a sequentially interactive protocol with sample complexity that is larger by only a small constant factor. By equivalent, we mean that for any $(\eps, 0)$-locally private compositional protocol, we can exhibit a sequentially interactive $(3\eps, 0)$-locally differentially private protocol with only a constant factor larger sample complexity that induces exactly the same distribution on transcripts. Thus for any task for which the original protocol was useful, the sequentially interactive protocol is just as useful\footnote{Formally, for any loss function defined over a data distribution $\cD$ and a transcript $\tra$, when data points $x_i$ are drawn i.i.d. from $\cD$, the two protocols induce exactly the same distribution over transcripts, and hence the same distribution over losses. Once one restricts attention to locally private protocols with privacy parameter $\epsilon \leq 1$ that take as input points drawn i.i.d. from a distribution $\cD$, it is without loss of generality to measure the success or failure of a protocol with respect to the underlying distribution $\cD$, rather than with respect to the sample. This is because such protocols are $\approx \epsilon/\sqrt{n}$ differentially private when viewed in the central model of differential privacy (in which the input may be permuted before used in the protocol) \cite{EFMRTT18,BBGN19}, and hence the distribution on transcripts would be almost unchanged even if the entire dataset was \emph{resampled} i.i.d. from $\cD$. \cite{CLNRW16,NRW18}. Thus, for such protocols, the transcript distribution is governed by the data distribution $\cD$, but not (significantly) by the sample.}.

More generally, we give a generic reduction under which any $(\eps, 0)$-private $k$-compositional protocol can be compiled into a sequentially interactive protocol with an $e^{\epsilon}k$-factor increase in sample complexity.

Our proof is constructive; given an arbitrary $k$-compositional $(\eps, 0)$-locally differentially private protocol we show how to simulate it using a sequentially interactive protocol that induces the same joint distribution on transcripts.  The ``simulation'' is driven by three main ideas:
\begin{enumerate}
\item \textbf{Bayesian Resampling}: The dataset used in a locally differentially private protocol is static once the protocol begins. However, we consider the following thought experiment: each user's datum is \emph{resampled} from the posterior distribution on their datum, conditioned on the transcript thus far, before every round in which they are given a local randomizer. We observe that the mechanism from this thought experiment induces exactly the same joint distribution on datasets and transcripts upon completion of the mechanism. Thus, for the remainder of the argument, we can seek to simulate this ``Bayesian Resampling'' version of the mechanism.
\item \textbf{Private Rejection Sampling}:  Because of the local differential privacy guarantee, at any step of the algorithm, the posterior on a user's datum conditioned on the private transcript generated so far must be close to their prior. Thus, it is possible to sample from this posterior distribution by first sampling from the prior, and then applying a rejection sampling step that is both a) likely to succeed, and b) differentially private. Sampling from the prior simply corresponds to querying a new user. At first glance, applying rejection sampling as needed seems to require information that the users will not have available, because they do not know the underlying data distribution $\cD$. But an application of Bayes rule, together with a data independent rescaling can be used to re-write the required rejection probability using only quantities that each user can compute from her own data point and the transcript. A similar use of rejection sampling appears in the simulation of locally private algorithms by statistical query algorithms given by~\citet{KLNRS11}.

\item \textbf{Data Independent Decomposition of Local Randomizers}: The two ideas above suffice to transform a fully interactive mechanism into a sequentially interactive mechanism, with a blowup in sample complexity from $n$ to $T$ (because in the sequentially interactive protocol that results from rejection sampling, each user applies only one local randomizer instead of an average of $T/n$). However, we generalize a recent result of \cite{BBGN19} to show that any $\epsilon_i$-private local randomizer can be described as a mixture between a \emph{data independent} distribution and an $(\eps, 0)$-private local randomizer for any $\epsilon > \epsilon_i$, where the weight on the data independent distribution is roughly (for small constant $\epsilon$) $1 - \epsilon_i/\epsilon$. Thus we can simulate each local randomizer while only needing to query a new user with probability $\epsilon_i/\epsilon$. As a result, for any compositional mechanism, 1 user in the sequential setting suffices (in expectation) to simulate the entire transcript of a single user in the fully interactive setting. More generally, if the mechanism is $k$-compositional, then $k$ users are required in expectation to carry out the simulation. The realized sample complexity concentrates sharply around its expectation.
\end{enumerate}

\subsection{Step 1: A Bayesian Thought Experiment}
The first step of our construction is to observe that for any locally private protocol $\A$, $\bexpt(\A)$ induces exactly the same distribution over transcripts as $\fexpt(\A)$. The difference is that in $\bexpt(\A)$, between each interaction with a given user $i$, their datum $x_i$ is \emph{resampled} from the posterior distribution on user $i$'s data conditioned on the portion of the transcript generated thus far. We prove in Lemma~\ref{lem:bayes} that the two experiments produce exactly the same transcript distribution. Once we establish this, our goal will be to simulate the transcript distribution induced by $\bexpt(\A)$.

\begin{algorithm}
\caption{}\label{alg:fipbayes}
\begin{algorithmic}[1]
\Procedure{$\bexpt$}{$\A, \cD, n$ }
\State Initialize transcript $\pi_0 = \emptyset$
\For{$t = 1,2, \ldots$}
	\If{$\A(\trar_{<t}) = \perp$}
		\State Output transcript $\trar_{<t}$
	\Else
		\State $(i^{t}, R_t, \eps_t,\delta_t) \gets \A(\trar_{<t})$
		\State Redraw $x_{i^t} \sim Q_{i,t}$ \Comment{$Q_{i,t}$ is the posterior on $x_{i^t}$ given $\trar_{ < t}$}
		\State User $i^{t}$ publishes $y_t \sim R_t(x_{i^t})$		
	\EndIf
\EndFor
\EndProcedure
\end{algorithmic}
\end{algorithm}

Note that when $i^{t}$ is selected for the first time $Q_{i,t} = \cD$, and so the sample complexity (e.g. number of draws from $\cD$) of $\bexpt(\A)$ is bounded by $n$.

\begin{lemma}
\label{lem:bayes}
	For any protocol $\A$, Let $\Pi^f$ be the transcript random variable that is output by $\fexpt(\A)$ and let $\Pi^b$ be the transcript output by $\bexpt(\A)$. Then $$ \Pi^f \stackrel{d}{=} \Pi^b$$ where $\stackrel{d}{=}$ denotes equality of distributions.
\end{lemma}
\begin{proof}
	We show this by (strong) induction on rounds in the transcript. The base case $t=1$ is immediate: for any index $i^{1}$ selected by $\bexpt(\A)$, the posterior distribution $Q_{i,1} $ is the same as the prior $\cD$.
	
	Now suppose it is true up to time $t+1$, i.e. $\Pi^f_{< t+1} \stackrel{d}{=} \Pi^b_{< t+1}$. Then since the joint distributions $\Pi_{ < t +2}$ factor as $(i^{t+1}, R_{t+1}, \epsilon_{t+1}, \delta_{t+1}, Y_{t+1}|\Pi_{ < t +1}) \cdot \Pi_{ < t +1}$, it suffices to show that the conditional distributions on $i^{t+1}, R_{t+1}, \epsilon_{t+1}, \delta_{t+1}, Y_{t+1}|\Pi_{ < t +1}$ coincide. Moreover, the conditional distribution on $i^{t+1}, R_{t+1}, \epsilon_{t+1}, \delta_{t+1}|\Pi_{ < t +1}$ is given by $\A(\Pi_{<{t+1}})$ under both algorithms, and so it remains only to show that $Y_{t+1}|i^{t+1}, R_{t+1}, \epsilon_{t+1}, \delta_{t+1},\Pi_{ < t +1}$ is the same distribution under both algorithms.
Under $\fexpt(\A)$,
	$$Y_{t+1}|i^{t+1}, R_{t+1}, \epsilon_{t+1}, \delta_{t+1},\Pi_{ < t +1} \sim R_{t+1}(x_{i^{t+1}}, \epsilon_{t+1}, \delta_{t+1}|\Pi_{ < t +1}) \stackrel{d}{=} R_{t+1}(u, \epsilon_{t+1}, \delta_{t+1}),$$
	where $u \stackrel{d}{=} x_{i^{t+1}}|\Pi_{< t+1} \stackrel{d}{=} Q_{i, t+1}$ by definition, and we use the fact that after conditioning on $\Pi_{<t+1}$, $x_{i^{t+1}}$ is independent of $\eps_{t+1}$ and $\delta_{t+1}$. Redrawing $u \sim Q_{i, t+1}$ does not change the marginal distribution of $R_{t+1}(u, \epsilon_{t+1}, \delta_{t+1})$, which is exactly the distribution under $\bexpt(\A)$, as desired.
\end{proof}

\subsection{Step 2: Sequential Simulation of Algorithm \ref{alg:fipbayes} via Rejection Sampling}
\label{subsec:rej}
We now show how to replace step $8$ in Algorithm \ref{alg:fipbayes} by selecting a new datapoint (drawn from $\cD$) at every round and using rejection sampling to simulate a draw from $Q_{i,t}$. The result is a sequentially interactive mechanism that preserves the transcript distribution of Algorithm \ref{alg:fipbayes} (and, by Lemma \ref{lem:bayes}, of Algorithm \ref{alg:fexpt}), albeit one with a potentially very large increase in sample complexity (from $n$ to $T$). The rejection sampling step increases the privacy cost of the protocol by at most a factor of $2$.

We first review why it is non-obvious that rejection sampling can be performed in this setting. We want to sample from the target distribution $Q_{i,t}$, the posterior $x_i^{t}|\pi_{< t}$, using samples from the proposal distribution $\cD$. Let $p_\pi$ denote the density function of $Q_{i,t}$ and let $p$ denote the density function of $\cD$. In rejection sampling, we would typically sample $u \sim \cD$, and with probability $\propto \frac{p_\pi(u)}{p(u)}$ we would accept $u$ as a sample drawn from $Q_{i,t}$, or else redraw another $u$ and continue.

This is not immediately possible in our setting, since the individuals (who must perform the rejection sampling computation) do not know the prior density $p$ and hence do not know the posterior $p_\pi$. As a result, they cannot compute either the numerator or denominator of the expression for the acceptance probability. We solve this problem by using the fact that we are simulating a posterior with a prior distribution, and formulate the rejection sampling probability ratio as a quantity depending only on a user's private data point and the transcript. Users may then compute this quantity themselves.

To define our transformed rejection sampler we set up some new notation: given a user $i$ and round $t$, let $\pi_{< t, i}$ denote the subset of the realized transcript up to time $t$ that corresponds to user $i$'s data, i.e. $\pi_{< t, i} = \{(i^{t'}, R_{t'}, \eps_{t'}, \delta_{t'}, y_{t'}): t' < t, i^{t'}= i\}$. Let $\P{x_i}{\pi_{< t, i}}$ denote the conditional probability of the messages corresponding to user $i$ given the choices of privacy parameters and randomizers up to time $t$: $$\P{}{\pi_{< t, i}} = \prod_{t' \colon i^{t'} = i}\P{R_{t'}}{R_{t'}(x_i, \epsilon_{t'}, \delta_{t'}) = y_{t'}}.$$ Using this notation, we define our rejection sampling procedure $\Rej$ in Algorithm~\ref{alg:RS}.

\begin{algorithm}
\caption{Rejection Sampling}\label{alg:RS}
\begin{algorithmic}[1]
\Procedure{$\Rej$}{$i, \pi_{ < t}, \eps, \eps_t, R_t(\cdot), \cD$} \Comment{Publishing $\Pi_{ < t}$ is $(\eps, 0)$-private}
\State Initialize indicator $\mathsf{accept} \gets 0$
\While{$\mathsf{accept} = 0$}
	\State Draw a new user $x \sim \cD$
	\State User $x$ computes $p_x \gets \frac{\P{x}{\pi_{< t, i}}}{\max_{x^*}\P{x^*}{\pi_{< t, i}}}$
	\State User $x$ publishes $\mathsf{accept} \sim \Ber{p_x/2}$
	\If{$\mathsf{accept} = 1$}
		\State User $x$ outputs $Y_t' \sim R_t(x, \eps_t)$
	\EndIf
\EndWhile
\EndProcedure
\end{algorithmic}
\end{algorithm}

We now prove that $\Rej$ is private and does not need to sample many users.

\begin{lemma}
\label{lem:rs}
	Let $Y_t  \stackrel{d}{=} R_t(x'),$ where $x' \sim Q_{i,t}$ and let $Y_t'$ be defined by the rejection sampling algorithm $\Rej$ above. Let the sample complexity $N$ be the total number of new users $x$ drawn in step 4 of $\Rej$. Then $\Rej$ is $(\eps + \eps_t, 0)$-locally private, $Y_t \stackrel{d}{=} Y_t'$, and $\mathbb{E}[N] \leq 2e^{\eps}$.
\end{lemma}
\begin{proof}[Proof of Lemma~\ref{lem:rs}]
\begin{claim}
	$\Rej$ is $(\eps+ \eps_t)$-locally private.
\end{claim}
 We first show that publishing a draw from $\Ber{p_x/2}$ is $(\eps, 0)$-locally private. By assumption publishing $\pi_{< t}$, and hence publishing $\pi_{< t, i}$ (by post-processing), $(\eps, 0)$-private. Hence for any $x \in \cX$ $$\P{}{1 \mid x} = p_x/2 = \frac{\P{x}{\pi_{< t, i}}}{2\max_{x^*}\P{x^*}{\pi_{< t, i}}} \in [1/(2e^{\eps}),1/2].$$ Therefore for any $x, x'$, $\P{}{1 \mid x} \leq e^{\eps} \P{}{1 \mid x'}$. Similarly,$$\P{}{0 \mid x} = (1 - p_x/2) \in [1/2, (2e^{\eps} - 1)/2e^{\eps}]$$ and by $1+x \leq e^x$, we get $1-\eps \leq e^{-\eps}$, so $1 + \eps \geq  2 - e^{-\eps}$, and $e^{\eps}/2 \geq (2^{\eps} - 1)/(2e^{\eps})$. Thus for any $x, x'$, $\P{}{0 \mid x} \leq e^{\eps} \P{}{0 \mid x'}$.

Releasing $R_t(x, \epsilon_t)$ is $\eps_t$-locally private, so by composition the whole process is $(\eps + \eps_t)$-locally private.

\begin{claim}
$Y_t \stackrel{d}{=} Y_t'$
\end{claim}
It suffices to show that $x |\{\mathsf{accept} = 1\} \stackrel{d}{=} Q_{i,t}$. Fix any $x_0 \in \X$. Then by Bayes' rule
\begin{align*}
	\P{}{x = x_0 \mid \mathsf{accept} = 1} =&\; \P{}{\mathsf{accept} = 1 \mid x = x_0} \cdot \frac{\P{}{x 	= x_0}}{\P{}{\mathsf{accept} = 1}} \\
	=& \; \frac{\P{x_0}{\pi_{< t, i}}}{\max_{x^*}\P{x^*}{\pi_{< t, i}}} \cdot \frac{\P{}{x = x_0}}	{\sum_{x'} \P{}{x = x'}\frac{\P{x'}{\pi_{< t, i}}}{\max_{x^*}\P{x^*}{\pi_{< t, i}}}}  \\
	=& \; \frac{\P{x_0}{\pi_{< t, i}}\P{}{x = x_0}}	{\sum_{x'} \P{}{x = x'}\P{x'}{\pi_{< t, i}}} \\
	=&\; \frac{\P{x_0}{\pi_{< t, i}}\P{}{x = x_0}}	{\P{}{\pi_{< t, i}}} \\
	=& \; \P{}{x = x_0 | \pi_{ < t, i}} \stackrel{d}{=} Q_{i,t},
\end{align*}
as desired. Finally, since $p_x/2 \geq \frac{1}{2e^{\eps}}$, the expected number of samples until $\mathsf{accept} = 1$ is $\leq 2e^{\eps}$.
\end{proof}

\subsection{Step 3: Data Independent Decomposition of Local Randomizers}
\label{subsec:decomp}
The preceding sections enable us to simulate a fully interactive $k$-compositional $(\eps,0)$-locally private protocol with a sequentially interactive $(2\eps,0)$-locally private protocol. However, our solution so far may require sampling a new user for each query in the original protocol. Since a fully interactive protocol's query complexity may greatly exceed its sample complexity, this is undesirable. To address this problem, we \emph{decompose} each local randomizer in a way that substantially reduces the number of queries that actually require samples.

Let $R: \X \to \Y$ be an $\epsilon'$ local randomizer, fix an arbitrary element $x_0 \in \X$, and let $x$ be a private input to $R$. Then Lemma 5.2 in \citet{BBGN19} shows that we can write $R(x)$ as a mixture $\gamma w + (1-\gamma)d_x$, where $w$ is a data-independent distribution, $d_x$ is a data-dependent distribution, and $\gamma \geq e^{-\epsilon'}$. This suggests that decomposition --- by answering a proportion of queries from data-independent distributions --- can reduce the sample complexity of our solution. Unfortunately, the data dependent distribution need not be differentially private (in fact, it often corresponds to a point mass on the private data point), so the privacy of the overall mechanism crucially relies on not releasing \emph{which} of the two mixture distributions the output was sampled from.

We first generalize this result, showing that for any $\epsilon \geq \epsilon'$, we can write $R(x)$ as $(1-\gamma) w + \gamma \tilde{R}(x)$ where $\tilde{R}$ is a $2\epsilon$-differentially private local randomizer, and $\gamma = \frac{e^{-\epsilon'}-1}{e^{-\eps}-1}$ (Lemma~\ref{lem:decomp}). The upshot of this generalization is that even if we make public which part of the mixture distribution was used, the resulting privacy loss is still bounded by $2\epsilon$. Larger values of $\epsilon$ increase our chance of sampling from a data-independent distribution when simulating a local randomizer, while increasing the privacy cost incurred by a user in the event that we sample from the  data-dependent mixture component. This tradeoff will be crucial for us in the proof of our main result.

\begin{lemma}[Data Independent Decomposition]
\label{lem:decomp}
	Let $R: \X \to \Y$ be an $\epsilon'$-differentially private local randomizer and let $\eps \geq \eps'$. Then there exists a mapping $\tilde{R}$ and fixed data-independent distribution $\mu$ such that $\tilde{R}(\cdot)$ is a $2\eps-$differentially private local randomizer and $$R(x) \stackrel{d}{=} \gamma \tilde{R}(x) + (1-\gamma)\mu, $$
where $\gamma = \frac{e^{-\epsilon'}-1}{e^{-\eps}-1}$.

\end{lemma}
\begin{proof}
Let $0 < \epsilon' \leq  \epsilon$, fix any $x_0 \in \X$, let $\gamma = \frac{e^{-\epsilon'}-1}{e^{-\epsilon}-1}$, and let $r(x)$ denote the density function of the local randomizer $R$ with input $x$ implicitly evaluated at some arbitrary point in the range, which we suppress. Since $\epsilon \geq \epsilon' > 0$, $\gamma \in (0, 1]$ is a valid mixture probability. Thus we can write $$r(x) = (r(x) - (1-\gamma)r(x_0)) + (1-\gamma)r(x_0)$$ and, rewriting the first term, $$r(x) - (1-\gamma)r(x_0) = \gamma (r(x_0) + \frac{1}{\gamma}(r(x)-r(x_0))) = \gamma \tilde{r}(x).$$ $\tilde r$ defines a new mapping $\tilde{R}(\cdot)$ by mapping $x$ to the random variable $\tilde{R}(x)$ with density function $\tilde{r}(x) = (r(x_0) + \frac{1}{\gamma}(r(x)-r(x_0)))$. Thus, it suffices to show that the mapping $\tilde{R}(x)$ is a $2\epsilon$-private local randomizer.

We first show that for any $x$, $\tilde{r}(x)$ is a well-defined density function. Since $R$ is an $\epsilon'$-private local randomizer, $r(x)-r(x_0) \geq (e^{-\epsilon'}-1)r(x_0)$, and so $$r(x_0) + \frac{1}{\gamma}(r(x)-r(x_0)) \geq r(x_0)\left(1 + \frac{e^{-\epsilon'}-1}{\gamma}\right) = r(x_0)e^{-\epsilon}.$$ This establishes that $\tilde{r}(x)$ is non-negative. Then since $$\int_{\Omega}\tilde{r}(x) = \int_\Omega r(x_0) + \frac{1}{\gamma}\int_\Omega(r(x)-r(x_0)) = 1 +  \frac{1}{\gamma}(1-1) = 1,$$ $\tilde{r}(x)$ defines a valid density function for any $x$.

To see that $\tilde r$ is also a $2\epsilon$-private local randomizer, fix any outcome $o \in \Y$ and any other $x' \in \X$. Since $r$ is an $\eps'$-local randomizer, $r(x) - r(x_0) \leq r(x_0)(e^{\eps'}-1)$ and we get

\begin{align*}
	\tilde r(x) =&\; r(x_0) + \frac{1}{\gamma}(r(x) - r(x_0)) \\
	\leq&\; r(x_0) \left[1 + \frac{1}{\gamma}\left(e^{\eps'}-1\right)\right] \\
	=&\; r(x_0)\left[1 + \frac{1 - e^{-\eps}}{1 - e^{-\eps'}}\left(e^{\eps'}-1\right)\right] \\
	=&\; r(x_0)\left[1 + e^{\eps'} \cdot \left(1 - e^{-\eps}\right)\right] \\
	\leq&\; r(x_0)\left[1 + e^{\eps} \cdot \left(1 - e^{-\eps}\right)\right] = r(x_0)e^\eps.
\end{align*}

We already showed $\tilde r(x') \geq e^{-\eps}r(x_0)$, so $$\frac{\tilde{r}(x)(o)}{\tilde{r}(x')(o)]} \leq \frac{e^{\epsilon}r(x_0)(o)}{e^{-\epsilon}r(x_0)(o)} \leq e^{2\epsilon}.$$
\end{proof}

\subsection{Putting it All Together: The Complete Simulation}
Finally, we combine rejection sampling and decomposition to give our complete reduction, Algorithm~\ref{alg:red}. We use rejection sampling to convert from a fully interactive mechanism to a sequentially interactive one and use our data-independent decomposition of local randomizers to reduce the sample complexity of the converted mechanism.

\begin{algorithm}
\caption{$\red$}\label{alg:red}
\begin{algorithmic}[1]
\Procedure{$\red$}{Fully interactive $(\eps, 0)-$LDP Protocol $\A,  \cD, n$}
\State Initialize $s_1, \ldots, s_n \gets 0$. \Comment{indicator if user $i$ has been selected yet}
\For{$t = 1 \ldots $}
	\If{$\A(\trar_{<t}) = \perp$}
		\State Output transcript $\trar_{<t}$
	\Else
		\State $(i^{t}, R_t, \eps_t) \gets \A(\trar_{<t})$
		\If{$s_i^{t} = 1$}
			\State Let $\gamma \gets \frac{e^{-\epsilon_t}-1}{e^{-\eps}-1}$
			\State Let $R_t = \gamma \tilde{R}_t + (1-\gamma)R_t(x_0)$ \Comment{Data Decomposition}
			\State Draw $\rho \sim \mathsf{Unif}(0,1)$
			\If{$\rho \leq \gamma$}
				 \State Draw $Y_t \sim \Rej(i^{t}, \pi_{ < t}, \eps, 2\eps, \tilde{R}(\cdot), \cD)$
			\Else
				\State Draw $Y_t \sim R_t(x_0, \epsilon_t)$ \Comment{Data independent distribution}
			\EndIf	
		\Else
			\State Draw $x_{i^{t}} \sim Q_{i,t} = \cD$, then draw $Y_t \sim R_t(x_{i^{t}}, \epsilon_t)$ \Comment{$Q_{i,t} = \cD$ since $s_{i^{t}} = 0$}
			\State Let $s_{i^{t}} \gets 1$
		\EndIf
	\EndIf
\EndFor
\EndProcedure
\end{algorithmic}
\end{algorithm}

We now prove that $\red$ has the desired interactivity, privacy, transcript, and sample complexity guarantees. We again denote by $N$ the sample complexity of $\red$, i.e. the number of samples drawn from the prior $\cD$ over the run of the algorithm, either in Step $15$ (which is bounded by $n$), or over the runs of $\Rej$ in line $10$. We observe that sampling from the prior $\cD$ simply corresponds to using a new datapoint drawn from $\cD$. Fixing a protocol $\A$, let $\Pi^r$ denote the transcript random variable generated by $\red(\A)$, and let $\Pi^b$ denote the transcript random variable generated by $\bexpt(\A)$.
\begin{theorem}
\label{thm:main}
	Let $\A$ a fully-interactive $k$-compositional $(\eps, 0)$-locally private protocol. Then
\begin{enumerate}
	\item $\red(\A)$ is sequentially interactive,
	\item $\red(\A)$ is $(3\eps,0)$-locally private,
	\item $\Pi^r \stackrel{d}{=} \Pi^b$,
	\item $\mathbb{E}[N] \leq n(\frac{2e^{\eps}\cdot \eps}{1-e^{-\eps}}k + 1)$, and with probability $1-\delta$, $N = O(nk + \sqrt{nk\log \frac{1}{\delta}})$.
\end{enumerate}
\end{theorem}

\begin{proof}[Proof of Theorem~\ref{thm:main}]
	\textbf{1. Interactivity:} Since each user $i$'s data is only used once (before $s_i$ is set to $1$), $\red$ is sequentially interactive.
	
	\textbf{2. Privacy:} Consider a data point $x$ corresponding to an arbitrary user over the run of $\red(\A)$. Then either $x$ is drawn in line $18$, or $x$ is drawn during a rejection sampling step. In the first case, $x$ is only used once in step $18$, as input to an $\eps_t$-local randomizer, preserving $(\eps, 0)$-LDP, since $\eps_t \leq \eps$. If $x$ is drawn during the rejection sampling step, then it is used during the use of rejection sampling to simulate a draw from a $(2\eps, 0)$-local randomizer $\tilde{R}(\cdot)$, where the input transcript $\pi_{< t}$ has been generated $(\eps, 0)$-privately. The privacy of the input transcript is relevant because it bounds the privacy of the user's rejection sampling step. By Lemma~\ref{lem:rs}, this is $(3\eps, 0)$-private.

	\textbf{3. Transcripts:} We prove this claim by a similar argument as that of Lemma~\ref{lem:bayes}: we show by induction that the transcript distribution at each step $t$ is the same for $\red(\A)$ and $\bexpt(\A)$. This is trivially true at $t = 1$. Now suppose it is true up to time $t+1$, i.e. $\Pi^r_{< t+1} \stackrel{d}{=} \Pi^b_{< t+1}$. Then since the joint distributions $\Pi_{ < t +2}$ factor as $(i^{t+1}, R_{t+1}, \epsilon_{t+1}, Y_{t+1}|\Pi_{ < t +1}) \cdot \Pi_{ < t +1}$, it suffices to show that the conditional distributions on $i^{t+1}, R_{t+1}, \epsilon_{t+1}, Y_{t+1}|\Pi_{ < t +1}$ coincide.
	
Note that under both $\red(\A)$ and $\bexpt(\A)$, protocol $\A$ is used to select $i^{t+1}, R_{t+1}, \epsilon_{t+1}$ as a function of $\Pi_{ < t+1}$,  so we can condition on $i^{t+1}, R_{t+1}, \epsilon_{t+1}$ as well, and need only show that the distribution on $Y_{t+1}$ is the same. Under $\bexpt(\A)$, $Y_{t+1}$ is drawn from $R_{t+1}(u, \epsilon_{t+1}), u \sim Q_{i, t+1}$. There are two cases for $\red(\A)$:
\begin{itemize}
\item If $s_i^{t+1} = 0$, then under $\red(\A), Y_{t+1}$ is drawn in line $18$ from $R_{t+1}(u, \epsilon_{t+1}), u \sim Q_{i, t+1}$, as desired.
\item  If $s_i^{t+1} = 1$, then $\red(\A)$ uses Lemma~\ref{lem:decomp} to write $R_{t+1}(\cdot)$ as a mixture. Hence if we sample from the mixture with input $u \sim Q_{i, t+1}$, we sample from $R_{t+1}(u)$, which is the desired sampling distribution. To see that $\red(\A)$ does sample from the target, we need only show that $Y_{t+1}$ drawn in line $13$ is sampled from $\tilde{R}_t(u)$ where $u \sim Q_{i, t+1}$. This is true by Lemma~\ref{lem:rs}.
\end{itemize}

	\textbf{4. Sample Complexity:} Here we bound the expected sample complexity,  deferring the high probability bound to Section~\ref{sec:app_highprob} in the Appendix. Let $N_i$ be the number of fresh samples drawn over all rounds $t$ where $i^t = i$, i.e. the number of samples drawn when simulating follow-up queries to $i$. Let $N_i^t$ be the number of samples drawn during rejection sampling in round $t$; we imagine that regardless of the coin-flip in line $11$ of the pseudocode of \red, $N_i^t$ is always drawn. Then the total number of samples is $N_i = \sum_{t =1}^{T}\gamma_tN_i^t$. (Note that for simplicity, we are summing over all rounds $T$, since equivalently we may imagine that each user is given a local randomizer at each round, with privacy cost $0$ in any round in which $i_t \neq i$.) Then by Lemma~\ref{lem:rs} $$\mathbb{E}[N_i] = \sum_{t =1}^{T}\gamma_t 2e^{\eps} = \frac{2e^{\eps}}{1 - e^{-\eps}}\sum_{t =1}^{T} (1 - e^{-\eps_t}).$$ Since $1-x \leq e^{-x}$, we get that $1 - \eps_t \leq e^{-\eps_t}$ and so $1 - e^{-\eps_t} \leq \eps_t$. Hence $$\mathbb{E}[N_i] \leq \frac{2e^{\eps}}{1-e^{-\eps}}\sum_{t=1}^{T}\epsilon_t \leq \left(\frac{2e^{\eps}\cdot \eps}{1-e^{-\eps}}\right)k.$$ Summing over $i$, and including the at most $n$ samples drawn in line $18$ bounds the expected sample complexity by $((\frac{2e^{\eps}\cdot \eps}{1-e^{-\eps}})k + 1)n$, as desired.
\end{proof}

%% file: lowerbound.tex
\section{Separating Full and Sequential Interactivity}
\label{sec:lb}
We now prove that our reduction in Section \ref{sec:simulation} is tight in the sense that any generic reduction from a fully interactive protocol to a sequentially interactive protocol must have a sample complexity blowup of $\tilde \Omega(k)$ when applied to a $k$-compositional protocol. Specifically, we define a family of problems such that for every $k$, there is a fully interactive $k$-compositional protocol that can solve the problem with sample complexity $n = n(k)$, but such that \emph{any} sequentially interactive protocol solving the problem must have sample complexity $\tilde \Omega(k \cdot n)$.

 Informally, the family of problems (Multi-Party Pointer Jumping, or $\mpj(d)$) we introduce is defined as follows. An \emph{instance} of $\mpj(d)$ is given by a complete tree of depth $d$. Every vertex of the tree is labelled by one of its children. By following the labels down the tree, starting at the root, an instance defines a unique root-to-leaf path. Given an instance of $\mpj(d)$, the data distribution is defined as follows: to sample a new user, first select a level of the tree uniformly $\ell \in [d]$ at random, and provide that user with the vertex-labels corresponding to level $\ell$ (note that fixing an instance of the problem, every user corresponding to the same level of the tree has the same data). The  problem we wish to solve privately is to identify the unique root to leaf path specified by the instance.

 We first show that there is a fully interactive protocol which can solve this problem with sample complexity $n = \tilde O(d^2/\epsilon^2)$. The protocol is $k$-compositional for $k = \Theta(d)$. Roughly speaking, the protocol works as follows: it identifies the path one vertex at a time, starting from the root, and proceeding to the leaf, in $d$ rounds. In each round, given the most recently identified vertex $v_i$ in level $\ell$, it attempts to identify the child that vertex $v_i$ is labelled with. It queries every user with the same local randomizer, which asks them to use randomized response to identify the labelled child of $v_i$ if their data corresponds to level $\ell$, and to respond with a uniformly random child otherwise (recall that the level that a user's data corresponds to is itself private, and hence is not known to the protocol). Since there are roughly $\tilde \Theta(\sqrt{n}/\epsilon^2)$ users with relevant data, out of $n$ users total, it is possible to identitify the child in question subject to local differential privacy. Although every user applies an $\epsilon$-local randomizer $d$ times in sequence, because each user's data corresponds to only a single level in the tree, the protocol is still $(\epsilon,0)$-locally private. Note that this privacy analysis mirrors the ``histogram'' structure of the non-compositional protocol in Example \ref{ex:comp}.

Informally, the reason that any sequentially interactive protocol must have sample complexity that is larger by a factor of $d$, is that even to identify the child of a single vertex in the local model, $\Omega(d^2/\epsilon^2)$ datapoints are required (this is exactly what our randomized response protocol achieves). But a sequentially interactive protocol cannot re-use these datapoints across levels of the tree, and so must expend $\Omega(d^2/\epsilon^2)$ samples for \emph{each} of the $d$ levels of the tree. This intuition is formalized in a delicate and technical induction on the depth of the tree, using information theoretic tools to bound the success probability of any protocol as a function of its sample complexity.  The precise definition of $\mpj(d)$ is somewhat more complicated, in which half of the weight on the underlying distribution is assigned to ``level 0'' dummy agents whose purpose is to break correlations between levels of the tree in the argument.

\def \D{\mathcal{D}}
\def \A{\mathcal{A}}
\def \1{\textbf{1}}

\subsection{The Multi-Party Pointer Jumping Problem}
We now formally define the \emph{Multi-party Pointer Jumping} ($\mpj$) problem.

\begin{definition}
Given integer parameter $d > 1$, an instance of Multi-party Pointer Jumping $\mpj(d)$ is defined by a vector $Z = Z_1 \circ \cdots \circ Z_d$, a concatenation of $d$ vectors of increasing length. Letting $s = d^4$, for each $i \in [d]$ $Z_i$ is a vector of $s^{i-1}$ integers in $\{0,1,\ldots,s-1\}$. For each $Z_i$, $Z_{i,j}$ is its $j^{th}$ coordinate.

Viewed as a tree, $Z$ is a complete $s$-ary tree of depth $d$ where each $Z_{i,j}$ marks a child of the $j$-th vertex at depth $i$. $P = P(Z)$ then denotes the vector of $d$ integers representing the unique root to leaf \emph{path} down this tree through the children marked by $Z$. Formally, $P$ is defined in a recursive way: $P_1 = Z_{1,1}$, ...,$P_i = Z_{i, P_1 \cdot s^{i-1} + P_2 \cdot s^{i-2} + \cdots + P_{i-1}+1}$,...,$P_d = Z_{d, P_1 \cdot s^{d-1} + P_2 \cdot s^{d-2} + \cdots + P_{d-1}+1}$.

Finally, an instance $\mpj(d)$ defines a data distribution $\cD$. For each $x \sim \cD$, with probability $1/2$, $x = (0,\emptyset)$ is a ``dummy datapoint'', and with the remaining probability $x = (\ell, Z_{\ell})$ where $\ell$ is a level drawn uniformly at random from $[d]$. A protocol solves $\mpj(d)$ if it recovers $P$ using samples from $\cD$.
\end{definition}

A graphical representation of $\mpj(d)$ where $s=2$ appears in Figure~\ref{fig:mpj} (we set $s=2$ in this figure for easier graphical representation).

\iffalse
For notation convenience, we use $Z_{|p_1,...,p_l}$ to denote the concatenation of $Z_{i,j}$'s such that $l+1 \leq i \leq d$ and $p_1 \cdot s^{d-1}+\cdots + p_l \cdot s^{d-l} < j \leq p_1 \cdot s^{d-1}+\cdots + p_l \cdot s^{d-l} + s^{d-l}$. This corresponds to the sub-tree defined by the partial path $p_1,\ldots,p_l$ of depth $l$ starting at the root.
\fi

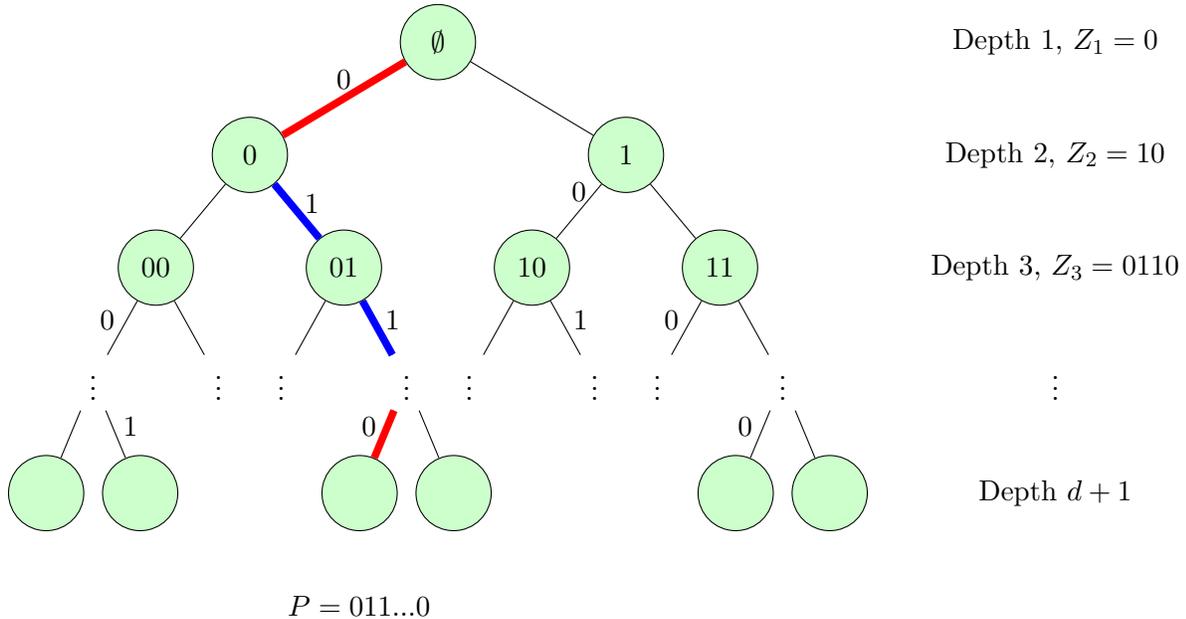
\begin{figure}[H]
\begin{tikzpicture}[level/.style={sibling distance=50mm/#1}]
\node [circle,draw,fill=green!20!white,minimum size=1cm] (z){$\emptyset$}
  child{node [circle,draw,fill=green!20!white,minimum size=1cm] (a) {0}
    child {node [circle,draw,fill=green!20!white,minimum size=1cm] (b) {00}
      child {node {$\vdots$}
        child {node [circle,draw,fill=green!20!white,minimum size=1cm] (d) {}}
        child {node [circle,draw,fill=green!20!white,minimum size=1cm] (e) {}
    edge from parent node[xshift=0.2cm,yshift = 0.1cm]{$1$}}
    edge from parent node[xshift=-0.2cm,yshift = 0.1cm]{$0$}
        }
      child {node {$\vdots$}}
    }
    child {node [circle,draw,fill=green!20!white,minimum size=1cm] (g) {01}
      child {node {$\vdots$}}
      child {node (cc){$\vdots$}
      child {node [circle,draw,fill=green!20!white,minimum size=1cm] (aa) {}
      		 child [grow=down] {node (aaa) {$P = 011...0$} edge from parent [draw=none]}
	      edge from parent node[xshift=-0.2cm,yshift = 0.1cm]{$0$}}
      child {node [circle,draw,fill=green!20!white,minimum size=1cm] (bb) {}}
    edge from parent node[xshift=0.2cm,yshift = 0.1cm]{$1$}}
    edge from parent node[xshift=0.2cm,yshift = 0.1cm]{$1$}
    }
    edge from parent node[above]{$0$}
  }
  child {node [circle,draw,fill=green!20!white,minimum size=1cm] (j) {1}
    child {node [circle,draw,fill=green!20!white,minimum size=1cm] (k) {10}
      child {node {$\vdots$}}
      child {node {$\vdots$}
    edge from parent node[xshift=0.2cm,yshift = 0.1cm]{$1$}}
    edge from parent node[above]{$0$}
    }
  child {node [circle,draw,fill=green!20!white,minimum size=1cm] (l) {11}
    child {node {$\vdots$}
    edge from parent node[xshift=-0.2cm,yshift = 0.1cm]{$0$}}
    child {node (c){$\vdots$}
      child {node [circle,draw,fill=green!20!white,minimum size=1cm] (o) {}
    edge from parent node[xshift = -0.2cm,yshift = 0.1cm]{$0$}}
      child {node [circle,draw,fill=green!20!white,minimum size=1cm] (p) {}
        child [grow=right] {node (q) {} edge from parent[draw=none]
          child [grow=right] {node (q) {Depth $d+1$} edge from parent[draw=none]
            child [grow=up] {node (r) {$\vdots$} edge from parent[draw=none]
              child [grow=up] {node (s) {Depth 3, $Z_3 = 0110$} edge from parent[draw=none]
                child [grow=up] {node (t) {Depth 2, $Z_2 = 10$} edge from parent[draw=none]
                  child [grow=up] {node (u) {Depth 1, $Z_1 =0$} edge from parent [draw=none]}
                }
              }
            }
          }
        }
      }
    }
  }
};
\path[draw,color=red,line width=1mm] (z)--(a);
\path[draw,color=blue,line width=1mm] (a)--(g);
\path[draw,color=blue,line width=1mm] (g)--(cc);
\path[draw,color=red,line width=1mm] (cc)--(aa);
\end{tikzpicture}
\caption{Multi-party Pointer Jumping}
\label{fig:mpj}
\end{figure}

\begin{algorithm}
\caption{A fully interactive $(\varepsilon,0)$-locally private protocol for $\mpj(d)$}
\label{alg:mpj}
\begin{algorithmic}[1]
\State Divide users into $u = \lceil \log(s)/\log(2) \rceil$ groups each of $m = 512d^2\log(d) \cdot\frac{(e^{\varepsilon}+1)^2}{(e^{\varepsilon}-1)^2}$ users.
\State Initialize $Q \gets 0$
\For{$r=1, 2, \ldots, d$}
	\State $Q_r \gets 0$
	\For{each group $g = 1, 2, \ldots, u$}
		\For{each user $i = 1, 2, \ldots, m$}
			\State $\ell_i \gets$ level of user $x_i$
			\If{$\ell_i = r$}
				\State $b_{i,r} \gets g$-th bit of binary representation of $Z_{r,Q+1}$
				\State User $i$ publishes randomized response $y_i \sim \RR{b_{i,r},\eps}$
			\Else
				\State User $i$ publishes $y_i \sim \Ber{0.5}$
			\EndIf
		\EndFor
		\State $g$-th bit of $Q_r \gets$ majority bit of $\{y_i\}_{i=1}^m$
	\EndFor
	\State $Q \gets s \cdot Q + Q_r$
\EndFor
\State Output $Q_1 \circ \cdots \circ Q_d$
\end{algorithmic}
\end{algorithm}

\subsection{An Upper Bound for Fully Interactive Mechanisms}

\begin{theorem}
There exists a fully interactive $(\varepsilon,0)$-locally private protocol (Algorithm \ref{alg:mpj}) with sample complexity $n = O(d^2 \log^2(d)(e^{\varepsilon}+1)^2/(e^\varepsilon-1)^2)$ that, on any instance $Z$ of $\mpj(d)$, correctly identifies $P(Z)$ with probability at least $1 - 1/d$.
\end{theorem}

\begin{proof}
First, it is easy to check that the total sample complexity of Algorithm \ref{alg:mpj} is $n = u \cdot m = O\left(d^2\log^2(d) \cdot\frac{(e^{\varepsilon}+1)^2}{(e^{\varepsilon}-1)^2}\right).$ Privacy follows from the same line of logic used in Example~\ref{ex:comp}: each agent sends $d$ bits in total, and at most  one of these bits is not sampled uniformly at random. Therefore, the probability of an agent sending any binary string of $d$ bits is bounded between $\frac{1}{2^{d-1}} \cdot \frac{1}{1 + e^{\varepsilon}}$ and $\frac{1}{2^{d-1}} \cdot \frac{e^{\varepsilon}}{1 + e^{\varepsilon}}$, for any datapoint that they might hold. Algorithm~\ref{alg:mpj} is therefore $(\varepsilon,0)$-locally private.

It remains to prove correctness. We first show that each group contains enough users from each level. For each group $g \in [u]$, define $X_{i,g,r}$ to be 1 if the $i$-th user in group $g$ has level $r$ and 0 otherwise. By definition, for any $r \in [d]$, $\P{}{X_{i,g,r}=1} = 1/(2d)$. Therefore we have $\E{}{\sum_{i=1}^m X_{i,g,r}} = m / (2d)$, and by a Chernoff bound
\[
\P{}{\sum_{i=1}^m X_{i,g,r} < m/(4d)} \leq \exp\left(-\frac{m}{16d}\right) \leq 1/(d^4).
\]
Define $W$ to be the event that for every $r \in [d]$ and $g \in [u]$, there are at least $m/(4d)$ users in group $g$ with level $r$. By a union bound, we know $\P{}{W} \geq 1- (ud)/d^4\geq 1-1/d^2$, so with high probability we have enough users in each level in each group.

We now analyze the quantities $Q_r$. For each $r \in [d]$, we want to show $$\P{}{Q_r = P_r |Q_1 = P_1,...,Q_{r-1} = P_{r-1},W} \geq 1-1 /d^3,$$ i.e. that the output $Q$ actually matches $P$.  Conditioning on $Q_1 = P_1, \ldots, Q_{r-1} = P_{r-1}$ and $W$, $Z_{r,Q+1} = P_r$. Define $Y_{i,g,r}$ to be 1 if the bit sent by the $i$-th user of group $g$ is equal to the $j$-th bit of $P_r$ and 0 otherwise. If the $i$-th user has level $r$ then they send their bit using randomized response and $\P{}{Y_{i,g,r} =1} = \frac{e^{\varepsilon}}{e^{\varepsilon}+1}$. If the $i$-th user's level is not $r$ then they send a uniform random bit $\P{}{Y_{i,g,r}=1}=1/2$. Since we conditioned on $W$, there are at least $m/(4d)$ users in group $g$ with level $r$. Thus
\begin{align}
	\E{}{\sum_{i=1}^m Y_{i,g,r}} =&\; \frac{m}{4d} \cdot \frac{e^\eps}{e^\eps+1} + \left(m - \frac{m}{4d}\right) \cdot \frac{1}{2} \label{eq:y_sum}.
\end{align}

Then we have	
\begin{align*}
	&\P{}{Q_r, P_r \text{ have the same }g\text{-th bit} |Q_1 = P_1,...,Q_{r-1} = P_{r-1},W}\\
=& \P{}{ \sum_{i=1}^m Y_{i,g,r} > \frac{m}{2}} \\
\geq & \P{}{\sum_{i=1}^m Y_{i,g,r} > \E{}{\sum_{i=1}^m Y_{i,g,r}} + \frac{m}{2} - \frac{m}{4d} \cdot \frac{e^\eps}{e^\eps+1} - \left(m - \frac{m}{4d}\right) \cdot \frac{1}{2}} ~~~~~~\text{(Equation~\ref{eq:y_sum})}\\
\geq& \P{}{\sum_{i=1}^m Y_{i,g,r} > \E{}{\sum_{i=1}^m Y_{i,g,r}}  - \frac{m}{8d}\cdot \frac{e^{\varepsilon}-1}{e^{\varepsilon}+1}}\\
\geq & 1 - \exp\left( - \frac{1}{2m} \cdot \left(\frac{m}{8d}\cdot \frac{e^{\varepsilon}-1}{e^{\varepsilon}+1} \right)^2\right)  ~~~~~~\text{(Chernoff bound)}\\
= & 1- \exp\left(- m \cdot \frac{1}{128d^2} \cdot \frac{(e^{\varepsilon}-1)^2}{(e^{\varepsilon}+1)^2} \right) \\
\geq & 1- \exp(-4\log(d)) = 1 -1/d^4.
\end{align*}

Union bounding over all $u$ groups yields $$\P{}{Q_r = P_r |Q_1 = P_1,...,Q_{r-1} = P_{r-1},W} \geq 1- u/d^4 \geq 1- 1/d^3.$$
Putting this all together, Algorithm \ref{alg:mpj} outputs $P(Z)$ with probability at least
\begin{align*}
	\P{}{Q_1 = P_1,...,Q_d = P_d} \geq&\; \P{}{W} \cdot \P{}{Q_1 = P_1,...,Q_d = P_d|W} \\
	\geq&\; \P{}{W} \prod_{r=1}^d \P{}{Q_r = P_r |Q_1 = P_1,...,Q_{r-1} = P_{r-1},W} \\
	\geq&\; (1-1/d^2)(1-1/d^3)^d \\
	\geq&\; 1-1/d.
\end{align*}
\end{proof}

Note that Algorithm \ref{alg:mpj} is $k$-compositional only for $k \geq \Omega(d)$. The lower bound that we prove next (Theorem \ref{thm:lower}) shows that any sequentially interactive protocol for the same problem must have a larger sample complexity by a factor of $\tilde \Omega(d) = \tilde \Omega(k)$, showing that in general, the sample-complexity dependence that our reduction (Theorem \ref{thm:main}) has on $k$ cannot be improved.

\subsection{A Lower Bound for Sequentially Interactive Mechanisms}

We prove our lower bound for sequentially interactive $(\eps,0)$-locally private protocols. As previous work~\cite{BNS18,CSUZZ18} has established that $(\eps,0)$- and $(\eps,\delta)$-local privacy are approximately equivalent for reasonable parameter ranges, our lower bound also holds for sequentially interactive $(\eps,\delta)$-locally private protocols. For an extended discussion of this equivalence, see Section~\ref{subsubsec:lb}.

\begin{theorem}
\label{thm:lower}
	Let $\A$ be a sequentially interactive $(\eps,0)$-locally private protocol that, for every instance $Z$ of $\mpj(d)$, correctly identifies $P(Z)$ with probability $\geq 2/3$. Then $\A$ must have sample complexity $n \geq  d^3/(216(e^\varepsilon-1)^2\log(d))$.
\end{theorem}

\begin{proof}
	We will prove that any sequentially interactive $(\varepsilon,0)$-locally private protocol with $n=d^3/(216(e^\varepsilon-1)^2\log(d))$ samples fails to solve $\mpj(d)$ correctly with probability $> 1/3$ when $Z$ is sampled uniformly randomly. This is a distributional lower bound which is only stronger than the theorem statement (a worst case lower bound). For notational simplicity,we assume in this argument that all local randomizers have discrete message spaces. However, this assumption is without loss of generality and can be removed (e.g. using the rejection sampling technique from~\citet{BS15}).

We will prove our lower bound even for protocols to which we ``reveal'' some information about the hidden instance $Z$ and users' inputs to the protocol and users. This only makes our lower bound stronger, as the mechanism can ignore this information if desired. Before the protocol starts, each user $i$ publishes a quantity $R_i$. $R_i = \ell_i$, user $i$'s level, if $\ell_i \neq 0$ (i.e., user $i$ is not a ``dummy'' user). Otherwise $R_i$ is set to be $\lfloor \frac{i -1 }{n/d}\rfloor+1$. At a high level, we reveal these $\{R_i\}_{i=1}^n$ to break the dependence between $Z_i$'s during the execution of the protocol (see Claim~\ref{clm:lb_prod} for a formalization of this intuition). Throughout the proof and its claims, we fix realizations $R_1 = r_1, R_2 = r_2, \ldots, R_n = r_n$. We will show that even given such $r_1,...,r_n$, any sequentially interactive $(\varepsilon,0)$-locally private protocol with $n$ users fails with probability more than $1/3$.

For each $i \in [n]$, denote by $\tra_i$ the message --- i.e., portion of the transcript --- sent by user $i$ via their local randomizer. Note that there is at most one such message since the protocol is sequentially interactive. We begin with a result about how conditioning on messages and revealed values affects the distribution of $Z$.

\begin{claim}
\label{clm:lb_prod}
	Suppose $Z_1,...,Z_d$ are sampled from a product distribution. Conditioned on the messages $\Pi_1,...,\Pi_i$ of the first $i$ users and the revealed values $R_1,\ldots,R_n$, $Z_1,...,Z_d$ are still distributed according to a product distribution.
\end{claim}
\begin{proof}
	We proceed via induction on the number of messages $i$. The base case $i=0$ is immediate from the assumption.  Now suppose the statement of the claim is true for $i-1$. Use $\D_1\times \D_2 \times \cdots \D_d$ to denote the product distribution of $Z_1,...,Z_d$ conditioned on $\Pi_1,...,\Pi_{i-1}$ and $R_1,\ldots,R_n$ (all quantities that follow are conditioned on $R_1,\ldots,R_n$, and so for notational simplicity we elide the explicit conditioning).
	
Since the protocol is sequentially interactive, conditioned on $\Pi_1,...,\Pi_{i-1}$, $\Pi_i$ depends only on $Z_{r_i}$, user $i$'s internal randomness, and their level $\ell_i$ (recall that when $r_i = \lfloor \frac{i -1 }{n/d}\rfloor+1$, it may be that $\ell_i = 0$ or $\ell_i = r_i$). Therefore, conditioned on $\Pi_1,...,\Pi_i$, $Z_1,...,Z_d$ distribute as $$\D_1\times \D_2 \times \cdots \times (\D_{r_i} | \Pi_i) \times \cdots \times \D_d,$$ a product distribution.
\end{proof}

We also use induction to prove the overall theorem. It has $d$ steps. For step $\ell \in [d]$, we let $\Delta_\ell$ be the following set of distributions on $Z$.
\begin{definition}[$\Delta_\ell$]
	For each $\ell \in [d]$, the set $\Delta_\ell$ is composed of distributions $\cD$ on $Z$ such that
	\begin{enumerate}
		\item $\cD$ is a product distribution on $Z_1,...,Z_d$,
		\item for each $i = 1,...,d-\ell$, $Z_i$ is deterministically fixed to be $z_i$, and
		\item since $Z_1,...,Z_{d-\ell}$ are fixed, by the definition of $\mpj$, $P_1,...,P_{d-\ell}$ are also fixed to some $p_1,...,p_{d-\ell}$. The marginal distribution on $Z_{|p_1,...,p_{d-\ell}}$ is the uniform distribution.
	\end{enumerate}
\end{definition}

In the induction step, we consider locally private  sequentially interactive protocols with fewer users. The idea is that for any sequentially interactive $(\eps,0)$-locally private protocol of $n$ users, if we fix the messages of the first $i$ users, then what remains is a sequentially interactive $(\eps,0)$-locally private protocol on $n-i$ users. Accordingly, we want to lower bound the failure probability of this remaining protocol. More concretely:

\paragraph{Inductive Statement}Any sequentially interactive $(\varepsilon,0)$-locally private protocol with $n\cdot \tfrac{\ell}{d}$ users (the $\left( n\cdot \tfrac{d-\ell}{d}+1\right)$-th user to the $n$-th user) fails to solve $\mpj(d)$ correctly with probability at $> 2/3 - \ell/(3d)$ when $Z$ is sampled from a distribution in $\Delta_\ell$.
\\\\It will be slightly more convenient to establish the inductive case first and then establish the base case second.

\paragraph{Induction step ($\ell>1$):}  Assume the above statement is true for $\ell-1$.

In this induction, let $\A$ be a sequentially interactive $(\varepsilon,0)$-locally private protocol with $n\cdot \tfrac{\ell}{d}$ users and let $\cD$ be the distribution generating $Z$ before $\A$ starts. Let $\Pi$ be the messages sent by the first $n/d$ users of $\A$ (the $\left( n\cdot \frac{d-\ell}{d}+1\right)$-th user to the $\left( n\cdot \frac{d-\ell+1}{d}\right)$-th user) and let $\A_{\pi}$ be the sequentially interactive $(\varepsilon,0)$-locally private protocol with $n\cdot \frac{\ell-1}{d}$ users conditioned on $\tra =\trar$. For notational convenience, define $n_\ell =n\cdot \frac{d-\ell}{d}$, $\tra_{<i} = \tra_{n_\ell+1},...,\tra_{i-1}$ and $\tra_{\leq i} =\tra_{n_\ell+1},...,\tra_i$.

For each prefix of messages, $\trar$, let $\D'(\trar)$ be some mixture of distributions in $\Delta_{\ell-1}$ (to be specified later). By the induction hypothesis on $\ell-1$, $$\P{Z\sim \D'(\pi)}{\A_{\pi} \text{ outputs } P(Z)} < 1/3 + (\ell-1)/(3d).$$
Thus
\begin{align*}
\P{Z \sim \D}{\A \text{ outputs } P(Z)} =&\; \sum_{\trar} \P{}{\tra=\trar} \cdot \P{Z\sim (\D|\tra=\trar)}{\A_{\trar} \text{ outputs } P(Z)}\\
	\leq&\; \sum_{\trar} \P{}{\tra=\trar} \cdot \left( \P{Z\sim \D'(\trar)}{\A_{\trar} \text{ outputs } P(Z)} + \|(\D|\tra=\trar) - \D'(\trar)\|_1\right) \\
	<&\; 1/3 + (\ell-1)/(3d) + \sum_{\trar} \P{}{\tra=\trar}\cdot \|(\D|(\tra=\trar)) - \D'(\trar)\|_1.
\end{align*}
It therefore suffices to show that $$\sum_{\pi} \P{}{\Pi=\pi}\cdot \|(\D|(\Pi=\pi)) - \D'(\pi)\|_1 \leq 1/(3d).$$ We show this via a sequence of three claims (Claims \ref{clm:lb_maxp}, \ref{clm:lb_info}, and~\ref{clm:lb_l1}), where $\D'(\trar)$ is defined in Claim~\ref{clm:lb_l1}.

First we define some notation for the path we need to reason about. Since $\D \in \Delta_\ell$, by the definition of $\Delta_\ell$ we know that for $Z \sim \D$, the first $d-\ell$ levels of the tree $Z_1,...,Z_{d-\ell}$ deterministically take fixed values $z_1,...,z_{d-\ell}$. Thus, the first $d-\ell$ nodes in the path $P_1,...,P_{d-\ell}$ are also fixed to take particular values $p_1,...,p_{d-\ell}$. For the induction step, we write $P=P_1,...,P_{d-\ell+1}$ to denote the first $d-\ell+1$ vertices of the path. Since $P_{d-\ell+1}$ is the only value that is not fixed, and the path is through an $s$-ary tree, $P$ can take on at most $s$ different possible values and is determined by  $Z_{d-\ell+1}$.

In the first claim, we show that after observing the messages sent by $n/d$ agents, there remains substantial uncertainty about $P$.
\begin{claim}
\label{clm:lb_maxp}
	For $i \in \{ n_\ell+1,...,n_\ell + n/d\}$, $$\sum_{\pi_{\leq i}} \P{}{\Pi_{\leq i}=\pi_{\leq i}} \cdot \left( \max_p \P{}{P = p|\Pi_{\leq i} = \pi_{\leq i}} \right)  \leq 3/d^4.$$
\end{claim}

\begin{proof}
Denoting by $\1_E$ the indicator function for event $E$,
\begin{align}
&\sum_{\pi_{\leq i}} \P{}{\Pi_{\leq i}=\pi_{\leq i}} \cdot \left( \max_p \P{}{P = p|\Pi_{\leq i} = \pi_{\leq i}} \right) \nonumber\\
\leq &\sum_{\pi_{\leq i}} \P{}{\Pi_{\leq i}=\pi_{\leq i}} \cdot\left(\1_{\max_p \P{}{P=p|\Pi_{\leq i} = \pi_{\leq i}}>2/s} \cdot 1 + \1_{\max_p \P{}{P = p|\Pi_{\leq i} = \pi_{\leq i}}\leq 2/s} \cdot \frac{2}{s} \right) \nonumber \\
\leq &\frac{2}{s} + \sum_{\pi_{\leq i}} \P{}{\Pi_{\leq i}=\pi_{\leq i}} \cdot \left(\1_{\max_p \P{}{P=p|\Pi_{\leq i} = \pi_{\leq i}}>2/s}  \right) \nonumber \\
\leq & \frac{2}{s} + \sum_p\sum_{\pi_{\leq i}} \P{}{\Pi_{\leq i}=\pi_{\leq i}} \cdot \left(\1_{ \P{}{P=p|\Pi_{\leq i} = \pi_{\leq i}}>2/s}  \right). \label{eq:triangle}
\end{align}

Now consider some specific $p$. We know that

\begin{align*}
\P{}{P=p|\Pi_{\leq i} = \pi_{\leq i}} =&\; \frac{\P{}{P= p, \Pi_{\leq i}= \pi_{\leq i}}}{\P{}{\Pi_{\leq i}=\pi_{\leq i}}} \\
	=&\; \P{}{P=p} \cdot \frac{\P{}{\Pi_{\leq i} = \pi_{\leq i}|P=p}}{\P{}{\Pi_{\leq i}=\pi_{\leq i}}}~~~~~~\text{(Bayes' rule)} \\
 	=&\; \frac{1}{s}  \cdot \frac{\P{}{\Pi_{\leq i} = \pi_{\leq i}|P=p}}{\P{}{\Pi_{\leq i} = \pi_{\leq i}}}~~~~~~\text{(uniformity of $P$)}.
\end{align*}

For $j = n_\ell+1,...,i$, define random variable $$X_j = \log \left(\frac{\P{}{\Pi_j|\Pi_{<j} , P=p}}{\P{}{\Pi_j|\Pi_{<j}}}\right).$$ As we ultimately want to upper bound the quantity in Equation~\ref{eq:triangle}, we now focus on bounding these $X_j$.

Recall that $r_j$ is user $j$'s level if that level is non-zero (i.e. user $j$ is not a ``dummy'' user). Otherwise $r_j$ is $d-\ell+1$ for $j = n_\ell+1,...,n_\ell+ n/d$. If $r_j \neq d-\ell+1$, by Claim \ref{clm:lb_prod}, we know that conditioned on $\Pi_{<j}$, $\Pi_j$ is independent of $P$. Therefore when $r_j \neq d - \ell + 1$, $X_j = \log(1) = 0$.

If instead $r_j = d-\ell+1$, we know the level $\ell_j$ of the user $j$ is $0$ with probability $d/(d+1)$ and $d-\ell+1$ with probability $1/(d+1)$. If $\ell_j = 0$, then the user is a ``dummy'' and since the user has no private data about $P$, $\Pi_j$ is independent of $P$ conditioned on $\Pi_{<j}$. Call the input distribution of the $j$-th user $q_j$. Here, we recall Lemmas 3 and 4 from~\citet{DJW13} and restate a simplified version as Lemma~\ref{lem:djws}

\begin{lemma}
\label{lem:djws}
	Let $m_1$ and $m_2$ be the output distributions of an $(\eps,0)$-local randomizer in a sequentially interactive protocol given, respectively, input distributions $q_j \mid \Pi_{<j}, P = p$ and $q_j \mid \Pi_{<j}$. Then $$\left| \log\left(\frac{m_1(z)}{m_2(z)}\right)\right| \leq \min(2,e^\eps)(e^\eps-1) \cdot \tv{(q_j \mid \Pi_{<j}, P = p)}{(q_j \mid \Pi_{<j})}.$$
\end{lemma}

We know that $\tv{(q_j|\Pi_{<j} = \pi_{<j},P=p)}{(q_j|\Pi_{<j} = \pi_{<j})} \leq 1/(d+1)$, as the difference stems from the event that $\ell_j = d - \ell + 1$. Thus, by Lemma~\ref{lem:djws} $$|X_j| \leq 2(e^{\varepsilon}-1)/(d+1) < 2(e^{\varepsilon}-1)/d.$$ Next, we bound the conditional expectation of $X_j$:

\begin{align*}
\E{}{X_j |\Pi_{<j} = \pi_{<j}} =&\; \sum_{\pi_j} \P{}{\Pi_j = \pi_j|\Pi_{<j} = \pi_{<j}}\cdot \log \left(\frac{\P{}{\Pi_j=\pi_j|\Pi_{<j}=\pi_{<j} , P=p}}{\P{}{\Pi_j=\pi_j|\Pi_{<j}=\pi_{<j}}}\right)\\
	=&\; -\kl{(\Pi_j|\Pi_{<j} = \pi_{<j},P=p)}{(\Pi_j|\Pi_{<j} = \pi_{<j})} \leq 0.
\end{align*}

Therefore $X_{n_\ell+1}, X_{n_\ell+1}+X_{n_\ell+2},...,X_{n_\ell+1}+ \cdots +X_{i}$ form a supermartingale. Next, we use the above bounds on these $X_j$ to control their sum using the Azuma-Hoeffding inequality:

\begin{align*}
\P{}{ X_{n_\ell+1}+ \cdots + X_i >\log(2)} \leq&\; \exp\left(-\frac{\log^2(2)}{2(2(e^{\varepsilon}-1)/d)^2(i-n_\ell)}\right)\\
	\leq&\; \exp\left(-\frac{\log^2(2)}{2(2(e^{\varepsilon}-1)/d)^2(n/d)}\right)\\
	\leq&\; \frac{1}{d^8} = \frac{1}{sd^4}.
\end{align*}

By the chain rule and Bayes' rule, we know $$X_{n_\ell+1}+ \cdots + X_i = \log\left(\frac{\P{}{\Pi_{\leq i}|P=p}}{\P{}{\Pi_{\leq i}}}\right) = \log\left(s\cdot \P{}{P=p|\Pi_{\leq i}}\right).$$ Therefore

\begin{align*}
	\sum_{\pi_{\leq i}} \P{}{\Pi_{\leq i}=\pi_{\leq i}} \cdot \left(\1_{ \P{}{P=p|\Pi_{\leq i} = \pi_{\leq i}}>2/s}  \right) =&\; \sum_{\pi_{\leq i}} \P{}{\Pi_{\leq i}=\pi_{\leq i}} \cdot \left(\1_{s\cdot \P{}{P=p|\Pi_{\leq i}= \pi_{\leq i}}>2}  \right) \\
	=&\; \P{}{ X_{n_\ell+1} + \cdots + X_i >\log(2)} \\
	\leq&\; \frac{1}{sd^4}.
\end{align*}

Tracing the above inequality back through Equation~\ref{eq:triangle}, we have
\begin{align*}
	\sum_{\pi_{\leq i}} \P{}{\Pi_{\leq i}=\pi_{\leq i}} \cdot \left( \max_e \P{}{P = p|\Pi_{\leq i} = \pi_{\leq i}} \right) \leq&\; \frac{2}{s} + \sum_p\sum_{\pi_{\leq i}} \P{}{\Pi_{\leq i}=\pi_{\leq i}} \cdot \left(\1_{ \P{}{P=p|\Pi_{\leq i} = \pi_{\leq i}}>2/s}  \right)\\
	\leq&\; \frac{2}{s} + s \cdot \frac{1}{sd^4} = \frac{3}{d^4}.
\end{align*}
\end{proof}

In Claim~\ref{clm:lb_info}, we bound the information $\Pi$ contains about $Z_{|P}$ (for a primer on information theory, see Appendix~\ref{sec:info}). Intuitively, by Claim~\ref{clm:lb_maxp} users have little information about $P$, and as a result they cannot know which potential subtree $Z_{|p}$ to focus their privacy budget on.

\begin{claim}
\label{clm:lb_info}
	$$\sum_{p} \P{}{P=p} \cdot I(\Pi;Z_{|p}|P=p) \leq 1/(18d^2).$$
\end{claim}
\begin{proof}
	By the inductive hypothesis, $Z$ is sampled from $\D \in \Delta_\ell$. Define $Z_{|<p}$ to be \\$Z_{|p_1,...,p_{d-\ell},0},...,Z_{|p_1,...,p_{d-\ell}, p_{d-\ell+1}-1}$. By the definition of $\Delta_l$, we know $Z_{|<p}$ and $Z_{|p}$ are independent given $P$, so $I(Z_{|<p};Z_{|p}|P=p)  = 0$. Therefore by the chain rule for mutual information, we get
\begin{align*}
 I(\Pi;Z_{|p}|P=p)  \leq  &I(\Pi,Z_{|<p};Z_{|p}|P=p) \\
 =& I(Z_{|<p};Z_{|p}|P=p) + I(\Pi;Z_{|p}|P=p,Z_{|<p})\\
  =&0 +  I(\Pi;Z_{|p}|P=p,Z_{|<p}).
\end{align*}

The main step of the proof of this claim is to compare $I(\Pi_i;Z_{|p}|P=p, \Pi_{<i}=\pi_{<i}, Z_{|<p})$ and $ I(\Pi_i;Z_{|p}| \Pi_{<i}=\pi_{<i}, Z_{|<p})$. First, by Claim \ref{clm:lb_prod}, conditioning on $\Pi_{<i} = \pi_{<i}$ induces a product distribution for $Z_1,...,Z_d$. We also know that (as mentioned in the proof of Claim~\ref{clm:lb_prod}) conditioned on $\Pi_{<i} = \pi_{<i}$, $\Pi_i$ only depends on $Z_{r_i}$, the internal randomness of the user $i$, and their level $\ell_i$. By item 3 in the definition of $\Delta_\ell$, $P$ only depends on $Z_{d-\ell+1}$. We prove

\begin{equation}
	I(\Pi_i;Z_{|p}|P=p, \Pi_{<i}=\pi_{<i}, Z_{|<p}) = I(\Pi_i;Z_{|p}| \Pi_{<i}=\pi_{<i}, Z_{|<p}). \label{eq:square}.
\end{equation}

There are two cases depending on $r_i$.
\begin{itemize}
	\item When $r_i  \leq d-\ell+1$, user $i$ either has $\ell_i \leq d - \ell + 1$ or is a ``dummy'' user. Therefore, whether or not we condition on $P=p$, $\Pi_i$ is independent of $Z_{|p},Z_{|<p}$. Thus $$I(\Pi_i;Z_{|p}|P=p, \Pi_{<i}=\pi_{<i}, Z_{|<p}) =0 =  I(\Pi_i;Z_{|p}| \Pi_{<i}=\pi_{<i}, Z_{|<p}).$$
	\item When $r_i > d-\ell+1$, once we've conditioned on $\Pi_{<i} = \pi_{<i}$, additionally conditioning on $P=p$ does not change the joint distribution of $Z_{d-\ell+2},...,Z_d$.  This is because $P = P_1, \ldots, P_{d - \ell + 1}$ and by above conditioning on $\Pi_{<i} = \pi_{<i}$ induces a product distribution on $Z_1, \ldots, Z_d$ (and in particular on $Z_{d-\ell+2},...,Z_d$). It follows that conditioning on $P=p$ does not change the joint distribution of $Z_{|p},Z_{|<p},\Pi_i$. Thus $$I(\Pi_i;Z_{|p}|P=p, \Pi_{<i}=\pi_{<i}, Z_{|<p}) = I(\Pi_i;Z_{|p}| \Pi_{<i}=\pi_{<i}, Z_{|<p}).$$
\end{itemize}

Putting things together, we have
\begin{align}
&\sum_{p} \P{}{P=p} \cdot I(\Pi;Z_{|p}|P=p) \nonumber \\
\leq&\sum_p \P{}{P=p} \cdot I(\Pi;Z_{|p}|P=p,Z_{|<p}) \nonumber \\
=& \sum_p \sum_{i =n_\ell+1}^{n_\ell+n/d} \P{}{P=p} \cdot I(\Pi_i;Z_{|p}|P=p, \Pi_{<i}, Z_{|<p}) \nonumber \\
=&\sum_{i =n_\ell+1}^{n_\ell+n/d} \sum_{\pi_{<i}}\sum_p\P{}{P=p} \cdot\P{}{\Pi_{<i} = \pi_{<i}|P=p}\cdot I(\Pi_i;Z_{|p}|P=p, \Pi_{<i}=\pi_{<i}, Z_{|<p}) \nonumber \\
=&\sum_{i =n_\ell+1}^{n_\ell+n/d}\sum_{\pi_{<i}}\sum_p\P{}{\Pi_{<i} = \pi_{<i}} \cdot\P{}{P=p|\Pi_{<i} = \pi_{<i}}\cdot I(\Pi_i;Z_{|p}|P=p, \Pi_{<i}=\pi_{<i}, Z_{|<p})~~~~~~\text{(Bayes' rule)} \nonumber \\
=&\sum_{i =n_\ell+1}^{n_\ell+n/d}\sum_{\pi_{<i}}\sum_p\P{}{\Pi_{<i} = \pi_{<i}} \cdot\P{}{P=p|\Pi_{<i} = \pi_{<i}}\cdot I(\Pi_i;Z_{|p}| \Pi_{<i}=\pi_{<i}, Z_{|<p})~~~~~~\text{(Equation~\ref{eq:square})} \nonumber \\
\leq&\sum_{i =n_\ell+1}^{n_\ell+n/d}\sum_{\pi_{<i}}\left(\sum_p\P{}{\Pi_{<i} = \pi_{<i}} \cdot I(\Pi_i;Z_{|p}| \Pi_{<i}=\pi_{<i}, Z_{|<p})\right)\cdot\left(\max_p\P{}{P=p|\Pi_{<i} = \pi_{<i}}\right) \nonumber \\
\leq&\sum_{i =n_\ell+1}^{n_\ell+n/d} \sum_{\pi_{<i}}\P{}{\Pi_{<i} = \pi_{<i}} \cdot I(\Pi_i;Z| \Pi_{<i}=\pi_{<i})\cdot\left(\max_p\P{}{P=p|\Pi_{<i} = \pi_{<i}}\right) \label{eq:bookmark}.
\end{align}

We now bound $I(\Pi_i;Z| \Pi_{<i}=\pi_{<i})$ using Theorem 1 from~\citet{DJW13}, simplified here as Lemma~\ref{lem:djw_2}.

\begin{lemma}
\label{lem:djw_2}
	Let $\Pi$ be the distribution over randomizer outputs for an $\eps$-local randomizer with inputs drawn from a distribution family parametrized by $\cV$. Then $I(\Pi; \cV) \leq 4(e^\eps-1)^2$.
\end{lemma}

In particular, the proof of Lemma~\ref{lem:djw_2} implies that $I(\Pi_i;Z| \Pi_{<i}=\pi_{<i}) \leq 4(e^\eps-1)^2.$ We continue our chain of inequalities.

\begin{align*}
	(\ref{eq:bookmark}) \leq&\; \sum_{i =n_\ell+1}^{n_\ell+n/d} \sum_{\pi_{<i}}\P{}{\Pi_{<i} = \pi_{<i}} \cdot 4(e^\varepsilon-1)^2 \cdot\left(\max_p\P{}{P=p|\Pi_{<i} = \pi_{<i}}\right)\\
	\leq&\; \frac{n}{d} \cdot (e^\varepsilon-1)^2 \cdot \frac{12}{d^4} ~~~\text{(Claim \ref{clm:lb_maxp})}\\
	\leq&\; \frac{1}{18d^2}.
\end{align*}
\end{proof}

In our last claim, we convert the bound on mutual information from Claim \ref{clm:lb_info} into a bound on the $L_1$ distance between distributions.
\begin{claim}
\label{clm:lb_l1}
There exists a distribution $\D'(\pi)$ which is a mixture of distributions in $\Delta_{\ell-1}$ for each $\pi$ such that $$\sum_{\pi} \Pr[\Pi=\pi] \cdot \|( \D|(\Pi=\pi)) -\D'(\pi) \|_1 \leq 1/(3d).$$
\end{claim}

\begin{proof}
By the definition of mutual information in terms of KL-divergence, $$I(\Pi; Z_{|p}|P=p) = \kl{\P{}{\Pi,Z_{|p}|P=p}}{\P{}{\Pi|P=p} \times \P{}{Z_{|p}|P=p}}.$$ Next, by Pinsker's inequality,

\begin{align*}
	&\;\sum_{\pi,z_{|p}} \left|\P{}{\Pi=\pi,Z_{|p}=z_{|p}|P=p} - \P{}{\Pi=\pi|P=p} \times \P{}{Z_{|p}=z_{|p}|P=p}\right| \\
	\leq&\; \sqrt{2 \kl{\P{}{\Pi,Z_{|p}|P=p}}{\P{}{\Pi|P=p} \times \P{}{Z_{|p}|P=p}}}
\end{align*}

\noindent so we may upper bound

\begin{align*}
&\sum_p \P{}{P=p} \cdot \sum_{\pi,z_{|p}} \left|\P{}{\Pi=\pi,Z_{|p}=z_{|p}|P=p} - \P{}{\Pi=\pi|P=p} \times \P{}{Z_{|p}=z_{|p}|P=p}\right|\\
 \leq &\sum_p \P{}{P=p} \cdot \sqrt{2 \kl{\P{}{\Pi,Z_{|p}|P=p}}{\P{}{\Pi|P=p} \times \P{}{Z_{|p}|P=p}}}\\
 = &\sum_p \P{}{P=p} \cdot \sqrt{2I(\Pi; Z_{|p}|P=p)}~~~~~~\text{(definition of mutual information)}\\
 \leq &  \sqrt{2\sum_p \P{}{P=p} \cdot 2I(\Pi; Z_{|p}|P=p)} ~~~~~\text{(Jensen's inequality and concavity of $\sqrt{\cdot}$)}\\
 \leq & 1/(3d) ~~~~~\text{(Claim \ref{clm:lb_info})}.
\end{align*}

On the other hand, we can also lower bound

\begin{align}
	&\;\sum_p \P{}{P=p} \cdot \sum_{\pi,z_{|p}} \left|\P{}{\Pi=\pi,Z_{|p}=z_{|p}|P=p} - \P{}{\Pi=\pi|P=p} \times \P{}{Z_{|p}=z_{|p}|P=p}\right| \nonumber \\
	=&\sum_p \P{}{P=p} \cdot \sum_{\pi} \P{}{\Pi=\pi|P=p} \cdot \sum_{z_{|p}} \left| \P{}{Z_{|p}=z_{|p}|\Pi=\pi,P=p}-\P{}{Z_{|p}=z_{|p}|P=p}\right| \nonumber \\
	=&\sum_{\pi} \P{}{\Pi=\pi} \cdot \sum_p \P{}{P=p|\Pi=\pi} \cdot \sum_{z_{|p}} \left| \P{}{Z_{|p}=z_{|p}|\Pi=\pi,P=p}-\P{}{Z_{|p}=z_{|p}|P=p}\right| \nonumber \\
	=&\sum_{\pi} \P{}{\Pi=\pi} \cdot \sum_p \P{}{P=p|\Pi=\pi} \cdot \nonumber \\
	& \sum_{z} \left| \P{}{Z = z|\Pi=\pi,P=p}-\P{}{Z_{|p}=z_{|p}|P=p} \cdot \P{}{Z=z|\Pi=\pi,P=p,Z_{|p} = z_{|p}}\right| \nonumber \\
	\geq&\sum_{\pi} \P{}{\Pi=\pi} \cdot \sum_{z} \big| \P{}{Z=z|\Pi=\pi}- \nonumber\\
	&\sum_p \P{}{P=p|\Pi=\pi} \cdot \P{}{Z_{|p}=z_{|p}|P=p}\cdot \P{}{Z=z|\Pi=\pi,P=p,Z_{|p} = z_{|p}}\big| \label{eq:swirl}
\end{align}

\noindent where the last equality comes from multiplying by $$1 = \sum_z \P{}{Z = z \mid \tra = \trar, P = p, Z_{\mid p} = z_{\mid p}}$$ and the inequality uses $$\P{}{Z = z \mid \tra = \trar} = \sum_p \P{}{Z = z \mid \tra = \trar, P = p} \cdot \P{}{P = p \mid \tra = \trar}$$ and the triangle inequality. With the preceding upper bound, the quantity in Equation~\ref{eq:swirl} is $\leq 1/(3d)$.

Now, define $\D'(\pi)$ to be the distribution on $Z$ such that for all $z$,
$$\P{Z\sim \D'(\pi)}{Z=z}=\sum_p\P{}{P=p|\Pi=\pi} \cdot \P{}{Z_{|p}=z_{|p}|P=p}\cdot \P{}{Z=z|\Pi=\pi,P=p,Z_{|p} = z_{|p}}.$$
Equivalently, $Z\sim \D'(\pi)$ is sampled through the following procedure: (1) sample $P$ according to $P \mid \tra = \trar$, (2) sample $Z_{|p}$ according to $Z_{\mid p} \mid P = p$, and (3) sample $Z$ according to $Z \mid \tra = \trar, P = p, Z_{\mid p} = z_{\mid p}$.

Noting that $\P{Z\sim \D|(\Pi=\pi)}{Z=z} = \P{}{Z=z|\Pi=\pi}$ for all $z$, since the quantity in Equation~\ref{eq:swirl} is $\leq 1/(3d)$ we get $$\sum_{\pi} \Pr[\Pi=\pi] \cdot \|( \D|(\Pi=\pi)) -\D'(\pi) \|_1 \leq 1/(3d).$$

It remains to show that $\D'(\pi)$ is a mixture of distributions in $\Delta_{\ell-1}$; doing so will complete our proof of the original inductive step. We will show for any $z_1,...,z_{d-\ell+1}$ such that $\P{Z\sim\D'(\pi)}{Z_1,...,Z_{d-\ell+1}=z_1,...,z_{d-\ell+1} }\neq 0$, $\D'(\pi) \mid (Z_1,...,Z_{d-\ell+1}=z_1,...,z_{d-\ell+1})$ is a distribution in $\Delta_{\ell-1}$. Recalling that membership in $\Delta_{\ell-1}$ requires meeting three conditions, we verify these conditions below.

\begin{enumerate}
	\item By Claim \ref{clm:lb_prod}, we know $\D|(\Pi=\pi)$ is a product distribution on $Z_1,...,Z_d$. It is easy to check that as $\D'(\pi)$ is sampled according to $\D|(\Pi=\pi)$, $\D'(\pi)$ is also a product distribution on $Z_1,...,Z_d$, and after the conditioning, $\D'(\pi)|(Z_1,...,Z_{d-\ell+1}=z_1,...,z_{d-\ell+1})$ remains a product distribution on $Z_1,...,Z_d$.
	\item Since we draw the final $Z$ conditioned on $Z_{\mid p} = z_{\mid p}$, $Z_i$ is deterministically fixed for $i = 1, \ldots, d - \ell$.
	\item First, note that the marginal distribution of $\D|(P=p)$ on $Z_{|p}$ is uniform since $D \mid (\tra = \trar)$ induces a product distribution on $Z_1, \ldots, Z_d$, and conditioning on $P = p$ only fixes $Z_{\leq d- \ell + 1}$ and leaves $Z_{d - \ell + 2} \times \cdots Z_d$ as a product distribution. Thus $$\P{Z \sim \D'(\pi)|(Z_{\leq d - \ell + 1} = z_{\leq d - \ell + 1})}{Z_{|p}  = z_{|p}} = \P{}{Z_{|p}=z_{|p} \mid P=p}$$
so the marginal distribution of $\D'(\pi)|(Z_1,...,Z_{d-\ell+1}=z_1,...,z_{d-\ell+1})$ on $Z_{|p}$ is also the uniform distribution. Therefore $\D'(\pi)|(Z_1,...,Z_{d-\ell+1}=z_1,...,z_{d-\ell+1})$ is a distribution in $\Delta_{\ell-1}$ and $\D'(\pi)$ is a mixture of distributions in $\Delta_{\ell-1}$.
\end{enumerate}

\end{proof}

\paragraph{Base case ($\ell = 1$):} We finally discuss the base case of our induction. Define $\A$, $\Pi$ and $P$ as in the induction step. Since the output of $\A$ is a function of $\Pi$, $$\P{}{\A \text{ outputs } P(Z)} \leq \sum_{\pi} \P{}{\Pi = \pi} \cdot \max_p \P{}{P=p | \Pi = \pi}.$$
Since Claim \ref{clm:lb_maxp} also applies to the base  case, we get $$
\P{}{\A \text{ outputs } P(Z)} \leq 3/d^4  < 1/3 < 1/3 + 1/(3d).$$
\end{proof} 

%% file: hypothesis.tex
\section{Hypothesis Testing}
\label{sec:hyp}
We now turn our attention to the role of interactivity in hypothesis testing. We first show that for the simple hypothesis testing problem, there exists a non-interactive $(\epsilon, 0)$-LDP protocol that achieves optimal sample complexity. This result extends to the compound hypothesis testing case, when we make the additional assumption that the sets of distributions are convex and compact.

\subsection{Simple Hypothesis Testing}
Let $P_0$ and $P_1$ be two known distributions such that $\tv{P_0}{P_1} \geq \alpha$, and suppose one of $P_0$ and $P_1$ generates $n$ i.i.d. samples $x_1, \ldots, x_n$. The goal in \emph{simple hypothesis testing} is to determine whether the samples are generated by $P_0$ or $P_1$. The Neyman-Pearson lemma~\cite{NP33} establishes that the likelihood ratio test is optimal for this problem absent privacy, and recent work~\cite{CKMSU18} extends this idea to give an optimal (up to constants) private simple hypothesis test in the centralized model of differential privacy. We recall a simple folklore non-interactive hypothesis test in the local model, and then prove that it is optimal even among the set of all fully interactive locally private tests.
\subsubsection{(Folklore) Upper Bound}
Consider the following simple variant $\A$ of the likelihood ratio test: each user $i$ with input $x_i$ outputs $\RR{\eps}{\arg \max_{j \in \{0,1\}} P_j(x_i)}$. For $j \in \{0,1\}$ let $\hat N_j$ denote the resulting count of responses and let $\hat N_j' = \tfrac{e^\eps+1}{e^\eps-1} \cdot \left(\hat N_j - \tfrac{n}{e^\eps-1} \right)$ be the corresponding de-biased count. The analyst computes both quantities $\hat N_j'$ and outputs $P_{\arg \max_j \hat N_j'}$.

\begin{algorithm}
\caption{Locally Private Simple Hypothesis Tester $\A$}\label{alg:lph}
	\begin{algorithmic}[1]
	\Procedure{Noninteractive Protocol}{$\{x_i\}_{i=1}^n$}
		\For{$i = 1 \ldots n$}
			\State User $i$ publishes $y_i \gets \RR{\arg \max_{j \in \{0,1\}} P_j(x_i), \eps}$
		\EndFor
		\For{$j = 0, 1$}
			\State Analyst computes $\hat N_j \gets |\{y_i \mid y_i = j\}|$
			\State Analyst computes $\hat N_j' \gets \tfrac{e^\eps+1}{e^\eps-1} \cdot \left(\hat N_j - \tfrac{n}{e^\eps-1} \right)$
		\EndFor
		\State Analyst outputs $P_{\arg \max_j \hat N_j'}$
	\EndProcedure
	\end{algorithmic}
\end{algorithm}

It is immediate that $\A$ is noninteractive and, since it relies on randomized response, satisfies $(\eps,0)$-local differential privacy. Since we can bound its sample complexity by simple concentration arguments, we defer the proof to Appendix~\ref{sec:app_hpub}.

\begin{theorem}
\label{thm:hp_ub}
	With probability at least $2/3$, $\A$ distinguishes between $P_0$ and $P_1$ given $n =  \Omega\left(\tfrac{1}{\eps^2\alpha^2}\right)$ samples.
\end{theorem}

\subsubsection{A Lower Bound for Arbitrarily Adaptive $(\eps,\delta)$-Locally Private Tests}
\label{subsubsec:lb}
We now show that the folklore $\epsilon$-private non-interactive test is optimal amongst all $(\epsilon,\delta)$-private fully interactive tests.   First, combining (slightly modified versions of) Theorem 6.1 from~\citet{BNS18} and Theorem A.1 in~\citet{CSUZZ18}, we get the following result\footnote{~\citet{BNS18} and~\citet{CSUZZ18} prove their results for noninteractive protocols. However, their constructions both rely on replacing a single $(\eps,\delta)$-local randomizer call for each user with an $(O(\eps),0)$-local randomizer call and proving that these randomizers induce similar output distributions. Since each user still makes a single randomizer call in sequential interactive protocols, essentially the same argument applies. For fully interactive protocols, a naive modification of the same result forces a stronger restriction on $\delta$, roughly $\delta = \tilde o\left(\frac{\eps \beta}{\max(n,T)}\right)$.}

\begin{lemma}
\label{lem:app_to_pure}
	Given $\eps > 0$, $\delta < \min\left(\tfrac{\epsilon\beta}{48n\ln(2n/\beta)}, \frac{\beta}{64n\ln(n/\beta)e^{7\eps}}\right)$ and sequentially interactive $(\eps,\delta)$-locally private protocol $\A$, there exists a sequentially interactive $(10\eps,0)$-locally private protocol $\A'$ such that for any dataset $U$, $\tv{\A(U)}{\A'(U)} \leq \beta$.
\end{lemma}

Lemma~\ref{lem:app_to_pure} enables us to apply existing lower bound tools for $\eps$-locally private protocols to (sequentially interactive) $(\eps,\delta)$-locally private protocols. At a high level, our proof relies on controlling the Hellinger distance between transcript distributions induced by an $(\eps,\delta)$-locally private protocol when samples are generated by $P_0$ and $P_1$. We borrow a simulation technique used by~\citet{BGMNW16} for a similar (non-private) problem and find that we can control this Hellinger distance by bounding the KL divergence between a simpler, \emph{noninteractive} pair of transcript distributions. We accomplish this last step using existing tools from~\citet{DJW13}.

\begin{theorem}
\label{thm:hypothesis_lb}
	Let $\tv{P_0}{P_1} = \alpha$ and let $\Pi$ be an arbitrary (possibly fully interactive) $(\eps,\delta)$-locally private simple hypothesis testing protocol distinguishing between $P_0$ and $P_1$ with probability $\geq 2/3$ using $n$ samples where $\eps > 0$ and $\delta < \min\left(\tfrac{\eps^3\alpha^2}{48n\ln(2n/\beta)}, \frac{\eps^2\alpha^2}{64n\ln(n/\beta)e^{7\eps}}\right)$. Then $n =  \Omega\left(\tfrac{1}{\eps^2 \alpha^2}\right).$
\end{theorem}
\begin{proof}
	Let $\Pi_{\vec 0}, \Pi_{\vec 1}$, and $\Pi_{\vec e_i}$ respectively denote the distribution over transcripts induced by protocol $\Pi$ when samples are drawn from $P_0$, $P_1$, and $x_i \sim P_1$ but the remaining $x_{i'} \sim P_0$. Let $h^2$ denote the square of the Hellinger distance, $\sqhel{f}{g} = 1 - \int_{\X} \sqrt{f(x)g(x)}dx$. We begin with Lemma~\ref{lem:bgmnw}, originally proven as Lemma 2 in~\citet{BGMNW16}.
	
\begin{lemma}
\label{lem:bgmnw}
	$\sqhel{\Pi_{\bar 0}}{\Pi_{\bar 1}} = O\left(\sum_{i=1}^n \sqhel{\Pi_{\bar 0}}{ \Pi_{\vec e_i}}\right)$.
\end{lemma}

Since our goal is now to bound these squared Hellinger distances, we will use a few facts collected below.

\begin{fact}
\label{fact:hell}
	For any distributions $f$, $g$, and $h$,
	\begin{enumerate}
		\item $\sqhel{f}{g} \leq 2(\sqhel{f}{h} + \sqhel{h}{g})$.
		\item $\sqhel{f}{g} \leq d_{TV}(f,g) \leq \sqrt{2}\hel{f}{g}$.
		\item $\sqhel{f}{g} \leq \tfrac{1}{2}\kl{f}{g}$.
	\end{enumerate}
\end{fact}

Choose an arbitrary term $i$ of the sum in Lemma~\ref{lem:bgmnw}. Suppose we have user $i$ simulate $\A$ using draws from $P_0$ for the inputs of other users and their input $x_i$ for input $i$. Since $\Pi$ is $(\eps,\delta)$-locally private, this simulation can be viewed as a single $(\eps,\delta)$-local randomizer applied to $x_i$. We can therefore use Lemma~\ref{lem:app_to_pure} to get a $(10\eps,0)$-local randomizer $\Pi'$ inducing distributions $\Pi_{\vec 0}'$ and $\Pi_{\vec e_i}'$ such that $\tv{\Pi_{\vec 0}'}{\Pi_{\vec 0}} \leq \eps^2\alpha^2$ and $\tv{\Pi_{\vec e_i}'}{\Pi_{\vec e_i}} \leq \eps^2\alpha^2$. Then,
\begin{align*}
	\sqhel{\Pi_{\vec 0}}{\Pi_{\vec e_i}} \leq&\; 2(\sqhel{\Pi_{\vec 0}}{\Pi_{\vec e_i}'} + \sqhel{\Pi_{\vec e_i}'}{\Pi_{\vec e_i}}) \\
	\leq&\; 4(\sqhel{\Pi_{\vec 0}}{\Pi_{\vec 0}'} + \sqhel{\Pi_{\vec 0}'}{\Pi_{\vec e_i}'}) + 2\sqhel{\Pi_{\vec e_i}'}{\Pi_{\vec e_i}} \\
	\leq&\; 4(\tv{\Pi_{\vec 0}}{\Pi_{\vec 0}'} + \sqhel{\Pi_{\vec 0}'}{\Pi_{\vec e_i}'}) + 2\tv{\Pi_{\vec e_i}'}{\Pi_{\vec e_i}} \\
	\leq&\; 6\eps^2\alpha^2 + 4\sqhel{\Pi_{\vec 0}'}{\Pi_{\vec e_i}'}
\end{align*}
where the first two inequalities follow from item 1 in Fact~\ref{fact:hell}, the third inequality follows from item 2, and the last inequality follows from our use of Lemma~\ref{lem:app_to_pure} above.

It remains to bound $\sqhel{\Pi_{\vec 0}'}{\Pi_{\vec e_i}'}$. By item 3 in Fact~\ref{fact:hell}, $4\sqhel{\Pi_{\vec 0}'}{\Pi_{\vec e_i}'} \leq 2\kl{\Pi_{\vec 0}'} {\Pi_{\vec e_i}'}$. Since the transcript distributions $\Pi_{\vec 0}'$ and $\Pi_{\vec e_i}'$ can be simulated by noninteractive $(10\eps,0)$-local randomizers, we can apply Theorem 1 from~\citet{DJW13}, restated for our setting as Lemma~\ref{lem:djw}.

\begin{lemma}
\label{lem:djw}
	Let $Q$ be an $(\eps,0)$-local randomizer and let $P_0$ and $P_1$ be distributions defined on common space $\cX$. Let $x_0 \sim P_0$ and $x_1 \sim P_1$. Then $$\kl{Q(x_0)}{Q(x_1)} + \kl{Q(x_1)}{Q(x_0)} \leq \min\{4,e^{2\eps}\}(e^\eps - 1)^2\tv{P_0}{P_1}^2.$$
\end{lemma}

\noindent Thus $$\kl{\Pi_{\vec 0}'}{\Pi_{\vec e_i}'} + \kl{\Pi_{\vec e_i}'}{\Pi_{\vec 0}'} = O(\eps^2 \cdot \tv{P_0}{P_1}^2) = O\left(\eps^2\alpha^2\right).$$ It follows that $\sqhel{\Pi_{\vec 0}'}{\Pi_{\vec e_i}'} = O(\eps^2\alpha^2)$. Moreover, since our original choice of $i$ was arbitrary, tracing back to Lemma~\ref{lem:bgmnw} yields $\sqhel{\Pi_{\vec 0}}{\Pi_{\vec 1}} = O(n\eps^2\alpha^2)$. By Fact~\ref{fact:hell}, $\sqhel{\Pi_{\vec 0}}{ \Pi_{\vec 1}} \geq \tfrac{1}{2}\tv{\Pi_{\vec 0}}{\Pi_{\vec 1}}^2 = \Omega(1)$. Thus $n = \Omega\left(\tfrac{1}{\eps^2\alpha^2}\right)$.
\end{proof}

\subsection{Compound Hypothesis Testing}
We now extend the reasoning of Section~\ref{sec:hyp} to \emph{compound} hypothesis testing. Here $P_0$ and $P_1$ are replaced by (disjoint) collections of discrete hypotheses $H_0$ and $H_1$ such that $$\inf_{(P,Q) \in H_0 \times H_1} \tv{P}{Q} \geq \alpha.$$ The goal is to determine whether samples are generated by a distribution in $H_0$ or one in $H_1$.

\begin{theorem}
\label{thm:compound}
	Let $H_0$ and $H_1$ be convex and compact sets of distributions over ground set $X$ such that $\inf_{(P,Q) \in H_0 \times H_1} \tv{P}{Q} \geq \alpha$. Then there exists noninteractive $(\eps,0)$-locally private protocol $\A$ that with probability at least $2/3$ distinguishes between $H_0$ and $H_1$ given $n =  \Omega\left(\tfrac{1}{\eps^2\alpha^2}\right)$ samples.
\end{theorem}
\begin{proof}
	Let $X$ be the ground set for distributions in $H_0$ and $H_1$, and consider the two-player zero-sum game $$\sup_{S \in \Delta(2^X)} \inf_{(P,Q) \in H_0 \times  H_1} \E{E \sim S}{P(E) - Q(E)}.$$ Here, the sup player chooses a distribution over events, and the inf player chooses distributions in $H_0$ and $H_1$. We will use (a simplified version of) Sion's minimax theorem~\cite{S58}.
	
\begin{lemma}[Sion's minimax theorem]
\label{lem:sion}
	For $f \colon A \times B \to \mathbb{R}$, if
	\begin{enumerate}
		\item for all $a \in A$ $f(a, \cdot)$ is continuous and concave on $B$,
		\item for all $b \in B$ $f(\cdot,b)$ is continuous and convex on $A$, and
		\item $A$ and $B$ are convex and $A$ is compact,
	\end{enumerate}
	then $$\sup_{b \in B} \inf_{a \in A} f(a,b) = \inf_{a \in A} \sup_{b \in B} f(a,b).$$
\end{lemma}

We first verify that the three conditions of Lemma~\ref{lem:sion} hold. Let $$f(S,(P,Q)) = \E{E \sim S}{P(E) - Q(E)}.$$ Linearity of expectation implies that $f(\cdot, (P,Q))$ is linear in $\Delta(2^X)$ and $f(S, \cdot)$ is linear in $H_0 \times H_1$. Therefore conditions 1 and 2 hold. Moreover, since $\Delta(2^X)$ is convex and we assumed $H_0$ and $H_1$ to be convex and compact --- properties which are both closed under Cartesian product --- condition 3 holds as well. As a result, $$\sup_{S \in \Delta(2^X)} \inf_{(P,Q) \in H_0 \times H_1} \E{E \sim S}{P(E) - Q(E)} = \inf_{(P,Q) \in H_0 \times H_1} \sup_{S \in \Delta(2^X)} \E{E \sim S}{P(E) - Q(E)} \geq \alpha$$ and there exists fixed distribution $S$ over events such that for all $(P,Q) \in H_0 \times H_1$, $$\E{E \sim S}{P(E) - Q(E)} \geq \alpha.$$

This leads to the following hypothesis testing protocol $\A$: for each $i \in [n]$, user $i$ computes $y_i = \E{E \sim S}{\1_{x_i \in E}}$ and publishes $y_i + \Lap{\tfrac{1}{\eps}}$. This protocol is immediately noninteractive, and since $y_i \in [0,1]$, this protocol is $(\eps,0)$-locally private over $\{x_i\}_{i=1}^n$. Finally, by the same analysis used to prove Theorem~\ref{thm:hp_ub} (replacing concentration of randomized responses with concentration of $\Lap{1}$ noise~\cite{CSS11}) it distinguishes between $H_0$ and $H_1$ with probability at least $2/3$ using $n = \Omega\left(\tfrac{1}{\eps^2\alpha^2}\right)$ samples.
\end{proof}

Since Theorem~\ref{thm:hypothesis_lb} still applies, this establishes that the above non-interactive protocol is also optimal.

%% file: appendix.tex
\section{Appendix}

\subsection{Properties of Differential Privacy}
\label{subsec:dp_props}
Differentially private computations enjoy two nice properties:
\begin{theorem}[Post Processing \cite{DMNS06}]
Let $A:\cX^*\rightarrow \mathcal{O}$ be any $(\eps,\delta)$-differentially private algorithm, and let $f:\mathcal{O}\rightarrow \mathcal{O'}$ be any function. Then the algorithm $f \circ A: \X^n\rightarrow \mathcal{O}'$ is also $(\eps,\delta)$-differentially private.
\end{theorem}
Post-processing implies that, for example, every \emph{decision} process based on the output of a differentially private algorithm is also differentially private.

\begin{theorem}[Basic Composition \cite{DMNS06}]
\label{thm:composition}
Let $A_1:\cX^*\rightarrow \mathcal{O}$, $A_2:\mathcal{O}\times \cX^*\rightarrow \mathcal{O}'$ be such that $A_1$ is $(\eps_1,\delta_1)$-differentially private, and $A_2(o,\cdot)$ is $(\eps_2,\delta_2)$-differentially private for every $o \in \mathcal{O}$. Then the algorithm $A:\cX^*\rightarrow \mathcal{O'}$ defined as $A(x) = A_2(A_1(x),x)$ is $(\eps_1+\eps_2,\delta_1+\delta_2)$-differentially private.
\end{theorem}

\subsection{Proof of Theorem~\ref{thm:hp_ub}}
\label{sec:app_hpub}
\begin{proof}
	For $j \in \{0,1\}$, let $N_j$ denote the (unknown) true count of 0 and 1 responses, i.e. $N_j = |\{x_i \mid \arg \max_{j' \in \{0,1\}} P_{j'}(x_i) = j\}|$. Then for both $j$, $\E{}{\hat N_j} = \frac{N_j(e^\eps-1) + n}{e^\eps + 1}$. By a Chernoff bound, with high probability $|\hat N_j - \tfrac{N_j(e^\eps-1) + n}{e^\eps+1}| = O(\sqrt{n})$. Then since $N_j' = \tfrac{e^\eps+1}{e^\eps-1} \cdot \left(\hat N_j - \tfrac{n}{e^\eps-1} \right)$ we get $|\hat N_j' - N_j| = O\left(\frac{e^{\eps+1}}{e^\eps-1} \cdot \sqrt{n}\right) = O\left(\frac{\sqrt{n}}{\eps}\right)$. It is therefore sufficient that $\tfrac{1}{\eps\sqrt{n}} = O(\alpha)$ to distinguish between $P_0$ and $P_1$, which implies the claim.
\end{proof}

\subsection{High Probability Sample Complexity from Theorem~\ref{thm:main}}
\label{sec:app_highprob}

We first prove a multiplicative Azuma-Hoeffding Inequality which will drive the high probability bound.
\begin{lemma} [Multiplicative Azuma-Hoeffding Inequality]
\label{lem:ah}
Let $(\gamma_t)_{t=1}^{T}, \gamma_{t} \in [0,1]$ a collection of dependent random variables, and let $(\F_t)_{t=1}^{T}$ a filtration such that $\sigma(\gamma_1, \ldots \gamma_{t-1}) \subset \F_{t-1}$. Suppose $\forall t, \E{}{\gamma_t \mid \F_{t-1}} \leq \mu_t$. Then if w.p. $1, \; \sum_{t}\mu_t \leq \mu,$ we have for any $\delta \in [e^{-3/4\mu}, 1]$:
$$\P{}{\sum_{t=1}^{T}\gamma_t > \sqrt{3\mu \log(1/\delta)} + \mu} \leq \delta$$
\end{lemma}
\begin{proof}
By convexity of $e^{l\gamma_t}, \gamma_t \in [0, 1],  \E{}{\gamma_t \mid \F_{t-1}} \leq \mu_t, \; \forall t, l$:
$$ \E{}{e^{l \gamma_t}\mid \F_{t-1}} \leq 1 + (e^{l}-1)\E{}{\gamma_t|\F_{t-1}} \leq 1 + (e^{l}-1)\mu \leq  e^{(e^{l}-1)\mu_t}$$
If we define $S_j = \sum_{t=1}^{j}\gamma_t$, then:
$$\E{}{e^{lS_j}} = \E{\F_{j-1}}{\E{}{e^{lS_j} \mid \F_{j-1}}} = \E{\F_{j-1}}{e^{lS_{j-1}}\mid F_{j-1}}\E{}{e^{l\gamma_j} \mid \F_{j-1}} \leq \E{}{e^{lS_{j-1}}}e^{(e^{l}-1)\mu_t} $$
Inducting on $j$, we have:
$$ \E{}{e^{lS_T}} \leq e^{(e^l-1)\sum_t \mu_t} \leq e^{(e^l-1)\mu} $$
For $\eps > 0$,  taking $l = \log(1+\eps), a = (1+\eps)\mu$ and using Markov's inequality:
$$\P{}{S_{T} \geq a} \leq e^{-(1+\eps)\mu l }\E{}{e^{lS}} \leq e^{-(1+\eps)\mu l + \mu(e^{l}-1)} = e^{-\mu \phi(\eps)},$$
where $\phi(z) = z - (1+z)\log(1+z)$. Since $\phi(z) \leq -z^2/3$ for $z \in [0, 3/2]$, we get: 
$$\P{}{S_{T} \geq (1+\eps)\mu}  \leq e^{-\mu \eps^2/3}$$
Setting $\eps = \sqrt{\frac{3\log(1/\delta)}{\mu}}$ gives the desired bound. Note that the condition $\eps \in [0, 3/2]$ forces $\delta \geq e^{-3/4\mu}$.

\end{proof}

\begin{proof}
There are at most $n$ users drawn in line $18$ of $\mathsf{Reduction}$, hence it suffices to bound with high probability the number of users drawn during rejection sampling steps in line $13$. For a given user $i$ drawn during a rejection sampling step, the sample complexity on rounds where $i$ is selected can be written as $\sum_{t: i_t = i}^{T}\gamma_tN_t$, where $\gamma_t \sim \mathsf{Ber}(\frac{e^{-\epsilon_{t}}-1}{e^{-\eps}-1})$, $N_t \stackrel{ind}{\sim} \mathsf{Geom}(p_{t})$ where $p_{t} \geq \frac{e^{-2\epsilon}}{2}, \eps_t \leq \eps$ are random variables that depend on $\Pi_{< t}$. Hence the total sample complexity over the rejection sampling rounds can be written as:
$$S = \sum_{t=1}^{T}\gamma_tN_t $$
First consider $\sum_{t = 1}^{T}\gamma_t$. Let $\F_{t}$ be the $\sigma$-algebra generated as $\F_t = \sigma(\Pi_{< t}, \eps_t, (\gamma_l)_{l=1}^{t-1})$. Then $\E{}{\gamma_t \mid \F_{t-1}} = \frac{e^{-\eps_t}-1}{e^{-\eps}-1} = \mu_t$.  Define $\mu = \sum_t \mu_t \leq \frac{n k \epsilon}{1-e^{-\eps}}$. Hence by Lemma~\ref{lem:ah}, with probability $1-\delta/2$, for $\delta \geq 2e^{-\frac{3}{4}\mu}$:

$$\P{}{\sum_{t=1}^{T}\gamma_t > \sqrt{3\mu \log(2/\delta)} + \mu} \leq \frac{\delta}{2}$$

Let $E_{\gamma}$ be the above event $\big\{\sum_{t}\gamma_t  \leq \sqrt{3\mu \log(2/\delta)} + \mu\big \}$. Then for any $t$, $$\mathbb{P}[S \geq t| E_{\gamma}] \leq \mathbb{P}[Z \geq t],$$
 where $Z = \sum_{t = 1}^{K}N_t', K = \sqrt{3\mu \log(2/\delta)} + \mu$, and $N_t' \stackrel{iid}{\sim} \mathsf{Geom}(\frac{e^{-\epsilon}}{2})$. Let $\mu' = \mathbb{E}[Z] = 2e^{\epsilon}(\sqrt{3\mu \log(2/\delta)} + \mu)$. By Theorem $2.1$ in \cite{boundexp} for any $t \geq 1$:

$$ \mathbb{P}[Z \geq t\mu'] \leq e^{\frac{-e^{-\epsilon}}{2}\mu' (t-1-\log t)}$$
Setting $t = 2(\frac{\log(2/\delta)2e^{\epsilon}}{\mu'} + 1)$ gives $Z \leq 2(\log(2/\delta)2e^{\epsilon} + \mu')$ with probability at least $1-\delta/2$. Hence $\P{}{S \geq 2(\log(2/\delta)2e^{\epsilon} + \mu')|E_{\gamma}} \leq \frac{\delta}{2}$.
Finally,
$$\P{}{S \geq 2(\log(2/\delta)2e^{\epsilon}+ \mu'} \leq \P{}{S \geq 2(\log(2/\delta)2e^{\epsilon}+ \mu'|E_{\gamma}}\P{}{E_{\gamma}} + (1-\P{}{E_{\gamma}}) \leq \frac{\delta}{2} + \frac{\delta}{2} = \delta$$
Substituting in the expression for $\mu'$ gives $S  = O(nk + \sqrt{nk\log \frac{1}{\delta}})$ with probability $1-\delta$,
 as desired.

 \end{proof}

%% file: prelim-info.tex
%TEX root = sifi.tex

\section{Information Theory}
\label{sec:info}

We briefly review some standard facts and definitions from information theory,
starting with entropy. Throughout, our $\log$ is base $e$.

\begin{definition}
The \emph{entropy} of a random variable $X$, denoted by $H(X)$, is defined as $H(X) = \sum_x \Pr[X = x] \log(1 / \Pr[X = x])$, and the \emph{conditional entropy} of random variable $X$ conditioned on random variable $Y$ is defined as $H(X|Y) = \mathbb{E}_y[H(X|Y = y)]$. 
\end{definition}

Next, we can use entropy to define the mutual information between two random
variables.

\begin{definition}
\label{def:muinfo}
The \emph{mutual information} between two random variables $X$ and $Y$ is defined as $I(X;Y) = H(X) - H(X|Y) = H(Y) - H(Y|X)$, and the \emph{conditional mutual information} between $X$ and $Y$ given $Z$ is defined as $I(X;Y|Z) = H(X|Z) - H(X|YZ) = H(Y|Z) - H(Y|XZ)$. 
\end{definition}

\begin{fact}\label{fact:cr}
Let $X_1,X_2,Y,Z$ be random variables, we have $I(X_1X_2;Y|Z) = I(X_1;Y|Z) + I(X_2;Y|X_1Z)$.
\end{fact}

\begin{definition}
The \emph{Kullback-Leibler divergence} between two random variables $X$ and $Y$ is defined as $\kl{X}{Y} = \sum_x \Pr[X = x] \log(\Pr[X = x] / \Pr[Y = x])$. 
\end{definition}

\begin{fact}
\label{fact:div}
Let $X,Y,Z$ be random variables, we have $$I(X;Y|Z) = \mathbb{E}_{x,z}[\kl{(Y| X = x, Z = z)}{(Y| Z = z)}].$$
\end{fact}

\begin{lemma}[Pinsker's inequality]
Let $X,Y$ be random variables, 
$$\sqrt{2\kl{X}{ Y}} \geq \sum_x |\Pr[X= x] - \Pr[Y=x]|$$.
\end{lemma}